\documentclass[lettersize,journal]{IEEEtran}
\IEEEoverridecommandlockouts  
\usepackage{times}
\usepackage{graphicx}
\usepackage{amsmath}
\usepackage{amssymb}
\usepackage[footnotesize]{caption}
\usepackage{mathtools}
\usepackage[normalem]{ulem}
\usepackage{algorithm}
\usepackage{gensymb}
\usepackage{algorithmic}
\usepackage{float,bm}
\usepackage{hyperref}
\usepackage[style=ieee, backend=biber]{biblatex}
\usepackage{hyperref}
\def\hyphenateAndTtWholeString #1{\xHyphenate#1$\wholeString\unskip}
\def\xHyphenate#1#2\wholeString {\if#1$%
    \else\transform{#1}%
    \takeTheRest#2\ofTheString\fi}
\def\takeTheRest#1\ofTheString\fi
{\fi \xHyphenate#1\wholeString}
\def\transform#1{\url{#1}\hskip 0pt plus 1pt}
\def\urlx #1{\href{#1}{\hyphenateAndTtWholeString{#1}}}

\usepackage{multicol}

\usepackage{amsmath} 

\usepackage{amsthm}
\usepackage{float}
\usepackage{multirow}
\usepackage{caption}
\usepackage{subcaption}

\usepackage[dvipsnames]{xcolor}
\usepackage{changes}

\providecommand{\R}{\ensuremath \mathbb{R}}
\providecommand{\N}{\ensuremath \mathbb{N}}

\newtheorem{defn}{Definition}
\newtheorem{rem}[defn]{Remark}
\newtheorem{lem}[defn]{Lemma}
\newtheorem{prop}[defn]{Proposition}
\newtheorem{assum}[defn]{Assumption}

\newtheorem{thm}[defn]{Theorem}

\providecommand{\methodname}{\text{REFINE}}
\providecommand{\Int}{\texttt{int}}

\providecommand{\Sum}{\texttt{sum}}

\providecommand{\Zaug}{\mathcal Z^{\text{aug}}}
\providecommand{\Zvel}{\mathcal Z^{\text{vel}}}
\providecommand{\zaug}{z^{\text{aug}}}
\providecommand{\zaugp}{z^{\text{aug}+}}
\providecommand{\zhi}{z^{\text{hi}}}
\providecommand{\zlo}{z^{\text{lo}}}
\providecommand{\zvel}{z^{\text{vel}}}
\providecommand{\zpos}{z^{\text{pos}}}
\providecommand{\cvel}{c^{\text{vel}}}
\providecommand{\Gvel}{G^{\text{vel}}}
\providecommand{\dzaug}{\dot z^{\text{aug}}}
\providecommand{\dzhi}{\dot z^{\text{hi}}}
\providecommand{\dzlo}{\dot z^{\text{lo}}}
\providecommand{\vlo}{v^{\text{lo}}}
\providecommand{\rlo}{r^{\text{lo}}}

\providecommand{\whlspd}{\omega_{\text{i}}}

\providecommand{\diag}{\texttt{diag}}
\providecommand{\slice}{\texttt{slice}}
\providecommand{\rot}{\texttt{rot}}
\providecommand{\ROT}{\texttt{ROT}}
\providecommand{\cost}{\texttt{cost}}
\providecommand{\Oego}{\mathcal O^\text{ego}}
\providecommand{\vegomax}{\nu^\text{ego}}
\providecommand{\vobsmax}{\nu^\text{obs}}
\providecommand{\tz}{t_0}
\providecommand{\tplan}{t_\text{plan}}
\providecommand{\tnb}{t_\text{m}}
\providecommand{\tf}{t_\text{f}}
\providecommand{\tb}{t_\text{brake}}
\providecommand{\tstop}{t_\text{stop}}
\providecommand{\tfstop}{t_\text{fstop}}
\providecommand{\tsmall}{t_\text{small}}
\providecommand{\udes}{u^\text{des}}
\providecommand{\dudes}{\dot u^\text{des}}
\providecommand{\ubrk}{u^\text{brake}}
\providecommand{\usmall}{u^\text{small}}
\providecommand{\eu}{e_u}
\providecommand{\deu}{\dot e_u}
\providecommand{\boff}{b_u^\text{off}}
\providecommand{\bpro}{b_u^\text{pro}}
\providecommand{\hdes}{h^\text{des}}
\providecommand{\rdes}{r^\text{des}}
\providecommand{\drdes}{\dot r^\text{des}}

\providecommand{\fhi}{f^{\text{hi}}}
\providecommand{\flo}{f^{\text{lo}}}
\providecommand{\amax}{a^{\text{dec}}}

\providecommand{\hmid}{h^\text{mid}}
\providecommand{\hrad}{h^\text{rad}}
\providecommand{\uc}{u^\text{cri}}
\providecommand{\lambdac}{\lambda^\text{cri}}
\providecommand{\alphac}{\alpha^\text{cri}}

\providecommand{\opt}{(\texttt{Opt}) }

\providecommand{\FRS}{\mathcal{F}_{xy}}

\providecommand{\Psparse}{\mathcal{P}^\text{sparse}}
\providecommand{\Pdense}{\mathcal{P}^\text{dense}}

\providecommand{\W}{\mathcal{W}}

\providecommand{\Z}{\mathcal{Z}}

\renewcommand{\P}{\mathcal{P}}

\providecommand{\B}{\mathcal{B}}
\providecommand{\Z}{\mathcal{Z}}

\providecommand{\T}{\mathcal{T}}

\providecommand{\RR}{\mathcal{R}}

\providecommand{\I}{\mathcal{I}}
\providecommand{\J}{\mathcal{J}}

\providecommand{\F}{\mathcal{F}}
\renewcommand{\SS}{\mathcal{S}}
\providecommand{\W}{\mathcal{W}}
\providecommand{\OO}{\mathcal{O}}

\providecommand{\FRS}{FRS}

\providecommand{\T}{\ensuremath T}

\newcommand{\zonocg}[2]{ \text{\textless} #1,\; #2 \text{\textgreater}}

\newcommand{\stkout}[1]{\ifmmode\text{\sout{\ensuremath{#1}}}\else\sout{#1}\fi}
\setdeletedmarkup{\color{cyan}\stkout{#1}}

\begin{document}

\title{\methodname: Reachability-based Trajectory Design using Robust Feedback Linearization and Zonotopes }
\author{Jinsun Liu$^*$, Yifei Shao$^*$, Lucas Lymburner, Hansen Qin, Vishrut Kaushik,\\ Lena Trang, Ruiyang Wang, Vladimir Ivanovic, H. Eric Tseng, and Ram Vasudevan
\thanks{Jinsun Liu, Lucas Lymburner, Vishrut Kaushik, Ruiyang Wang, and Ram Vasudevan are with the Department of Robotics, University of Michigan, Ann Arbor, MI 48109. \texttt{\{jinsunl, llyburn, vishrutk, ruiyangw, ramv\}@umich.edu}.}
\thanks{Yifei Shao is with the Department of Computer and Information Science, University of Pennsylvania, Philadelphia, PA 19104. \texttt{yishao@seas.upenn.edu}.}
\thanks{Hansen Qin is with the Department of Mechanical Engineering, University of Michigan, Ann Arbor, MI 48109. \texttt{qinh@umich.edu}.}
\thanks{Lena Trang is with the College of Engineering, Ann Arbor, MI 48109. \texttt{ltrang@umich.edu}.}
\thanks{Vladimir Ivanovic and Eric Tseng are with Ford Motor Company. \texttt{\{vivanovi, htseng\}@ford.com}. }
\thanks{This work is supported by the Ford Motor Company via the Ford-UM Alliance under award N022977.}
\thanks{$*$These two authors contributed equally to this work.}
}

\maketitle

\begin{abstract}
Performing real-time receding horizon motion planning for autonomous vehicles while providing safety guarantees remains difficult.
This is because existing methods to accurately predict ego vehicle behavior under a chosen controller use online numerical integration that requires a fine time discretization and thereby adversely affects real-time performance.
To address this limitation, several recent papers have proposed to apply offline reachability analysis to conservatively predict the behavior of the ego vehicle.
This reachable set can be constructed by utilizing a simplified model whose behavior is assumed \textit{a priori} to conservatively bound the dynamics of a full-order model.
However, guaranteeing that one satisfies this assumption is challenging. 
This paper proposes a framework named \methodname{} to overcome the limitations of these existing approaches.
\methodname{} utilizes a parameterized robust controller that partially linearizes the vehicle dynamics even in the presence of modeling error.
Zonotope-based reachability analysis is then performed on the closed-loop, full-order vehicle dynamics to compute the corresponding control-parameterized, over-approximate Forward Reachable Sets (FRS).
Because reachability analysis is applied to the full-order model, the potential conservativeness introduced by using a simplified model is avoided.
The pre-computed, control-parameterized FRS is then used online in an optimization framework to ensure safety.
The proposed method is compared to several state of the art methods during a simulation-based evaluation on a full-size vehicle model and is evaluated on a $\frac{1}{10}$th race car robot in real hardware testing.
In contrast to existing methods, \methodname{} is shown to enable the vehicle to safely navigate itself through complex environments. 
\end{abstract}

\begin{IEEEkeywords}
Motion and path planning, robot safety, reachability analysis, control, zonotopes.
\end{IEEEkeywords}


\IEEEpeerreviewmaketitle

\section{Introduction}
\begin{figure}[t]
    \centering
    \includegraphics[{trim=0cm, 10cm, 16.5cm, 0cm}, width=1.0\columnwidth,clip=true]{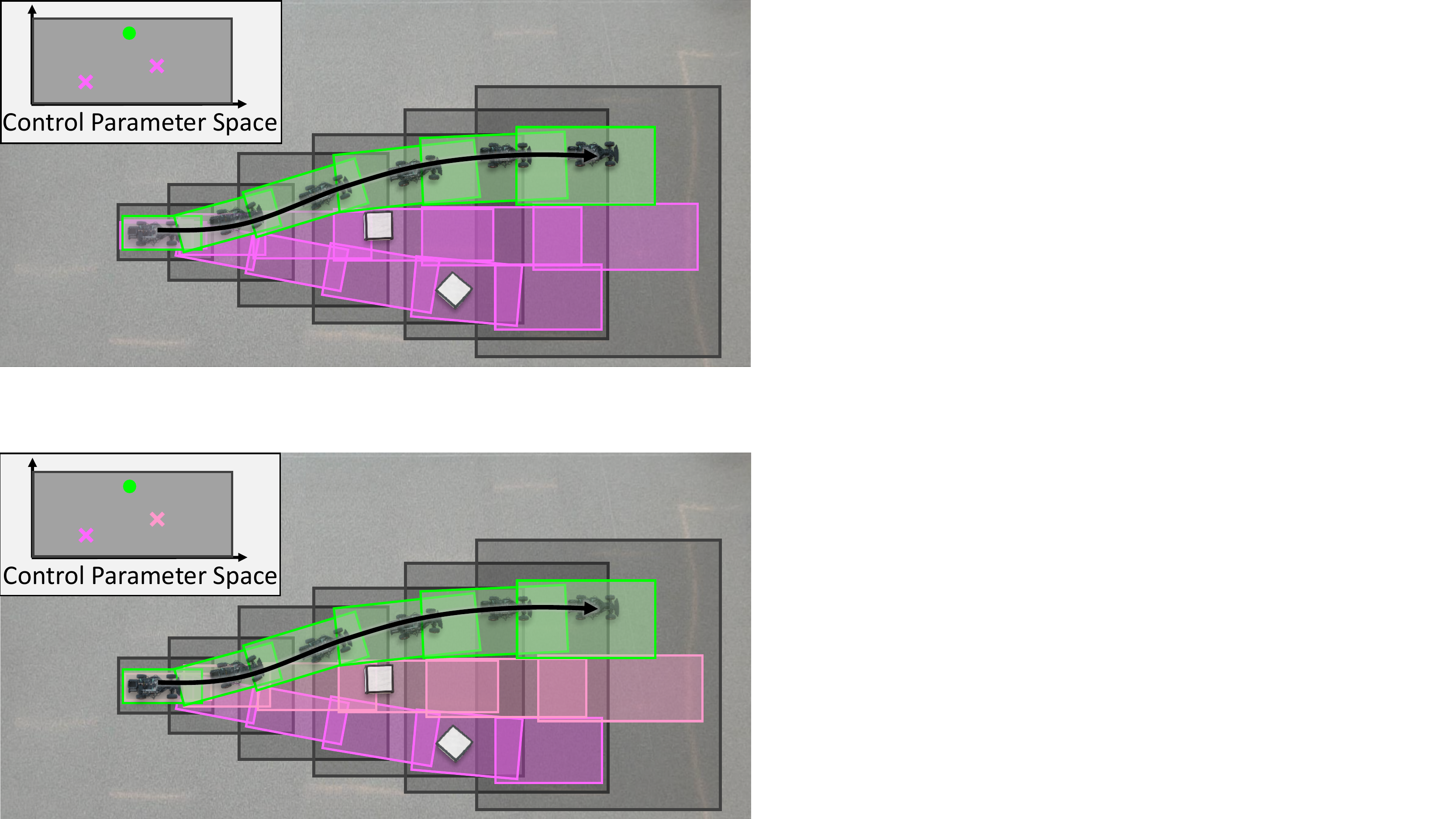}
    \caption{
    \methodname{} first designs a robust controller to track parameterized reference trajectories by feedback linearizing a subset of vehicle states.
    \methodname{} then performs offline reachability analysis using a closed-loop full-order vehicle dynamics to construct a control-parameterized, zonotope reachable sets (shown as grey boxes) that over-approximate all possible behaviors of the vehicle model over the planning horizon. 
    During online planning, \methodname{} computes a parameterized controller that can be safely applied to the vehicle by solving an optimization problem, which selects subsets of pre-computed zonotope reachable sets that are guaranteed to be collision free.  
    In this figure, subsets of grey zonotope reachable sets corresponding to the control parameter shown in green ensures a collision-free path while the other two control parameters shown magenta 
    might lead to collisions with white obstacles.
    }
    \label{fig: high level}
    \vspace*{-0.5cm}
\end{figure}
Autonomous vehicles are expected to operate safely in unknown environments with limited sensing horizons.
Because new sensor information is received while the autonomous vehicle is moving, it is vital to plan trajectories using a receding-horizon strategy in which the vehicle plans a new trajectory while executing the trajectory computed in the previous planning iteration. 
It is desirable for such motion planning frameworks to satisfy three properties:
First, they should ensure that any computed trajectory is dynamically realizable by the vehicle.
Second, they should operate in real time so that they can react to newly acquired environmental information collected. 
Finally, they should verify that any computed trajectory when realized by the vehicle does not give rise to collisions.
This paper develops an algorithm to satisfy these three requirements by designing a robust, partial feedback linearization controller and performing zonotope-based reachability analysis on a full-order vehicle model.

We begin by summarizing related works on trajectory planning and discuss their potential abilities to ensure safe performance of the vehicle in real-time.
To generate safe motion plan in real-time while satisfying vehicle dynamics, it is critical to have accurate predictions of vehicle behavior over the time horizon in which planning is occurring. 
Because vehicle dynamics are nonlinear, closed-form solutions of vehicle trajectories are incomputable and approximations to the vehicle dynamics are utilized.
For example, sampling-based methods typically discretize the system dynamic model or state space to explore the environment and find a path, which reaches the goal location and is optimal with respect to a user-specified cost function \cite{lavalle2001randomized,janson2015fast}.
To model vehicle dynamics during real-time planning, sampling-based methods apply online numerical integration and buffer obstacles to compensate for numerical integration error \cite{elbanhawi2014sampling,lavalle2006planning,kuwata2009real}. 
Ensuring that a numerically integrated trajectory can be dynamically realized and be collision-free can require applying fine time discretization.
This typically results in an undesirable trade-off between these two properties and real-time operation. 
Similarly, Nonlinear Model Predictive Control (NMPC) uses time discretization to generate an approximation of solution to the vehicle dynamics that is embedded in optimization program to compute a control input that is dynamically realizable while avoiding obstacles \cite{howard2007optimal, falcone2008low, urmson2008autonomous, wurts2018collision}. 
Just as in the case of sampling-based methods, NMPC also suffers from the undesirable trade-off between safety and real-time operation.

To avoid this undesirable trade-off, researchers have begun to apply  reachability-based analysis.
Traditionally reachability analysis was applied to verify that a pre-computed trajectory could be executed safely \cite{althoff2014online, pek2020using}. 
More recent techniques apply offline reachable set analysis to compute an over-approximation of the Forward Reachable Set (FRS), which collects all possible behaviors of the vehicle dynamics over a fixed-time horizon.
Unfortunately computing this FRS is challenging for systems that are nonlinear or high dimensional.
To address this challenge, these reachability-based techniques have focused on pre-specifying a set of maneuvers and simplifying the dynamics under consideration.
For instance, the funnel library method \cite{majumdar2017funnel} computes a finite library of funnels for different maneuvers and over approximates the FRS of the corresponding maneuver by applying Sums-of-Squares (SOS) Programming.
Computing a rich enough library of maneuvers and FRS to operate in complex environments can be challenging and result in high memory consumption.
To avoid using a finite number of maneuvers, a more recent method called Reachability-based Trajectory Design (RTD) was proposed \cite{kousik2020bridging} that considers a continuum of trajectories and applies SOS programming to represent the FRS of a dynamical system as a polynomial level set.
This polynomial level set representation can be formulated as functions of time for collision checking \cite{kousik2017safe,vaskov2019not,vaskov2019towards}.
Although such polynomial approximation of the FRS ensures strict vehicle safety guarantees while maintaining online computational efficiency, SOS optimization still struggles with high dimensional systems.
As a result, RTD still relies on using a simplified, low-dimensional nonlinear model that is assumed to bound the behavior of a full-order vehicle model. 
Unfortunately it is difficult to ensure that this assumption is satisfied. 
More troublingly, this assumption can make the computed FRS overly conservative because the high dimensional properties of the full-order model are treated as disturbances within the simplified model.

These aforementioned reachability-based approaches still pre-specify a set of trajectories for the offline reachability analysis.
To overcome this issue, recent work has applied a Hamilton-Jacobi-Bellman based-approach \cite{bansal2017hamilton}  to pose the offline reachability analysis as a
differential game between a full-order model and
a simplified planning model \cite{herbert2017fastrack}.
The reachability analysis computes the tracking error between the full-order and planning models, and an associated controller to keep the error within the computed bound at run-time.
At run-time, one buffers obstacles by this bound, then ensures that the planning model can only plan outside of the buffered obstacles. 
This approach can be too conservative in practice because
the planning model is treated as if it is trying to escape from the high-fidelity model.

To address the limitations of existing approaches, this paper proposes a real-time, receding-horizon motion planning algorithm  named REchability-based trajectory design using robust Feedback lInearization and zoNotopEs (\methodname{}) depicted in Figure \ref{fig: high level} that builds on the reachability-based approach developed in \cite{kousik2020bridging} by using feedback linearization and zonotopes.
This papers contributions are three-fold:
First, a novel parameterized robust controller that partially linearizes the vehicle dynamics even in the presence of modeling error.
Second, a method to perform zonotope-based reachability analysis on a closed-loop, full-order vehicle dynamics to compute a control-parameterized, over-approximate Forward Reachable Sets (FRS) that describes the vehicle behavior.
Because reachability analysis is applied to the full-order model, potential conservativeness introduced by using a simplified model is avoided.
Finally, an online planning framework that performs control synthesis in a receding horizon fashion by solving optimization problems in which the offline computed FRS approximation is used to check against collisions.
This control synthesis framework applies to All-Wheel, Front-Wheel, or Rear-Wheel-Drive vehicle models.

The rest of this manuscript is organized as follows:
Section \ref{sec: prelim} describes necessary preliminaries and Section \ref{sec: dynamics} describes the dynamics of Front-Wheel-Drive vehicles.
Section \ref{sec: trajectory design and safety} explains the trajectory design and vehicle safety in considered dynamic environments.
Section \ref{sec: controller} formulates the robust partial feedback linearization controller.
Section \ref{sec: rtd} describes Reachability-based Trajectory Design and how to perform offline reachability analysis using zonotopes.
Section \ref{sec: online} formulates the online planning using an optimization program,
and in Section \ref{sec:implementation} the proposed method is extended to various perspectives including All-Wheel-Drive and Rear-Wheel-Drive vehicle models.
Section \ref{sec:experiment} describes how the proposed method is evaluated and compared to other state of the art methods in simulation and in hardware demo on a 1/10th race car model.
And Section \ref{sec: conclusion} concludes the paper.

\section{Preliminaries}
\label{sec: prelim}
This section defines notations and set representations that are used throughout the remainder of this manuscript.
Sets and subspaces are typeset using calligraphic font.
Subscripts are primarily used as an index or to describe an particular coordinate of a vector.

Let $\R$, $\R_+$ and $\N$ denote the spaces of real numbers, real positive numbers, and natural numbers, respectively.
Let $0_{n_1\times n_2}$ denote the $n_1$-by-$n_2$ zero matrix.
The Minkowski sum between two sets $\mathcal A$ and $\mathcal A'$ is $\mathcal A\oplus \mathcal A' = \{a+a'\mid a\in \mathcal A, ~a'\in \mathcal A'\}$.
The power set of a set $\mathcal A$ is denoted by $P(\mathcal A)$.
Given vectors $\alpha,\beta\in\R^n$, let $[\alpha]_i$ denote the $i$-th element of $\alpha$, 
let $\Sum(\alpha)$ denote the summation of all elements of $\alpha$, let $\|\alpha\|$ denote the Euclidean norm of $\alpha$,
let $\diag(\alpha)$ denote the diagonal matrix with $\alpha$ on the diagonal, and let $\Int(\alpha,\beta)$ denote the $n$-dimensional box $\{\gamma \in\R^n\mid [\alpha]_i\leq[\gamma]_i\leq[\beta]_i,~\forall i=1,\ldots,n\}$.
Given $\alpha \in \R^n$ and $\epsilon > 0 $, let $\mathcal{B}(\alpha, \epsilon)$ denote the $n$-dimensional closed ball with center $\alpha$ and radius $\epsilon$ under the Euclidean norm.
Given arbitrary matrix $A\in\R^{n_1\times n_2}$, let $A^\top$ be the transpose of $A$, let $[A]_{i:}$ and $[A]_{:i}$ denote the $i$-th row and column of $A$ for any $i$ respectively, and let $|A|$ be the matrix computed by taking the absolute value of every element in $A$.

Next, we introduce a subclass of polytopes, called zonotopes, that are used throughout this paper:
\begin{defn}
\label{def: zonotope}
A \emph{zonotope} $\Z$ is a subset of $\R^n$ defined as
\begin{equation}
    \mathcal Z = \left\{x\in\R^n\mid x= c+\sum_{k=1}^\ell \beta_k g_k, \quad \beta_k \in [-1,1] \right\}
\end{equation}
with \emph{center} $c\in\R^n$ and $\ell$ \emph{generators} $g_1,\ldots,g_\ell\in\R^n$.
For convenience, we denote $\mathcal Z$ as $\zonocg{c}{G}$ where $G = [g_1, g_2, \ldots, g_\ell ]\in\R^{n\times\ell}$.
\end{defn}
\noindent Note that an $n$-dimensional box is a zonotope because
\begin{equation}
\label{eq: interval is zonotope}
    \Int(\alpha,\beta) = \zonocg{\frac{1}{2}(\alpha+\beta)}{\frac{1}{2}\diag(\beta-\alpha)}.    
\end{equation}
By definition the Minkowski sum of two arbitrary zonotopes  $\Z_1 = \zonocg{c_1}{G_1}$ and $\Z_2=\zonocg{c_2}{G_2}$ is still a zonotope as $\Z_1\oplus\Z_2 = \zonocg{c_1+c_2}{[G_1,G_2]}$.
Finally, one can define the multiplication of a matrix $A$ of appropriate size with a zonotope $\Z=\zonocg{c}{G}$ as
\begin{equation}
\label{eq: zono-matrix mult}
   A \Z = \left\{x\in\R^n\mid x= A c+\sum_{k=1}^\ell \beta_k A g_k, ~ \beta_k \in [-1,1] \right\}.
\end{equation}
Note in particular that $A \Z$ is equal to the zonotope $\zonocg{A c}{A G}$.


\section{Vehicle Dynamics}
\label{sec: dynamics}

This section describes the vehicle models that we used in both high-speed and low-speed scenarios throughout this manuscript for autonomous navigation with safety concerns.

\subsection{Vehicle Model}

\begin{figure}[t]
    \centering
    \includegraphics[trim={0cm, 0cm, 20cm, 8cm}, clip,width=0.7\columnwidth,clip=true]{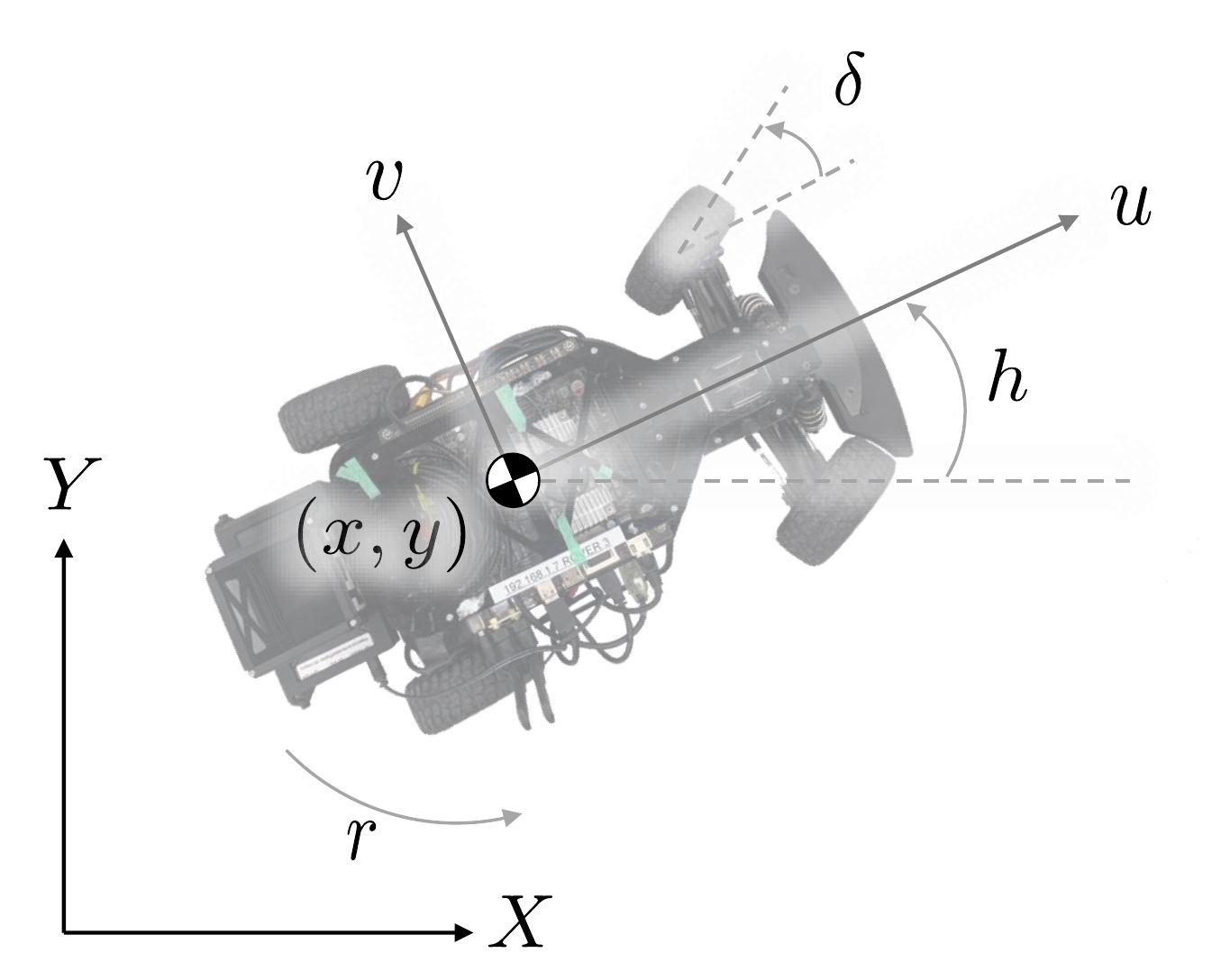}
    \caption{Vehicle model with the global frame shown in black and body frame in gray.}
    \label{fig: vehicle model}
    \vspace*{-0.5cm}
\end{figure}

The approach described in this paper can be applied to a front-wheel-drive (FWD), rear-wheel drive (RWD), or all-wheel drive (AWD) vehicle models. 
However, to simplify exposition, we focus on how the approach applies to FWD vehicles and describe how to extend the approach to AWD or RWD vehicles in Section \ref{sec: AWD}.
To simplify exposition, we attach a body-fixed coordinate frame in the horizontal plane to the vehicle as shown in Fig. \ref{fig: vehicle model}. 
This body frame's origin is the center of mass of the vehicle, and its axes are aligned with the longitudinal and lateral directions of the vehicle.
Let $\zhi(t) = [x(t),y(t),h(t),u(t),v(t),r(t)]^\top\in \R^6$  be the states of the vehicle model at time $t$, where $x(t)$ and $y(t)$ are the position of vehicle's center of mass in the world frame, $h(t)$ is the heading of the vehicle in the world frame, $u(t)$ and $v(t)$ are the longitudinal and lateral speeds of the vehicle in its body frame, $r(t)$ is the yaw rate of the vehicle center of mass, and $\delta(t)$ is the steering angle of the front tire. 
To simplify exposition, we assume vehicle weight is uniformly distributed and ignore the aerodynamic effect while modeling the flat ground motion of the vehicles by the following dynamics \cite[Chapter 10.4]{jazar2008vehicle}: 
\begin{equation}
\renewcommand*{\arraystretch}{1.2} 
    \label{eq:highspeed_perfect}
     \dzhi(t)= \begin{bmatrix} \dot{x}(t) \\ \dot{y}(t) \\ \dot{h}(t) \\ \dot{u}(t) \\ \dot{v}(t) \\ \dot{r}(t)
     \end{bmatrix} = 
    \begin{bmatrix}
        u(t)\cos h(t) - v(t)\sin h(t)\\ u(t)\sin h(t) + v(t)\cos h(t) \\ r(t)\\
        \frac{1}{m}\big(F_\text{xf}(t)+F_\text{xr}(t)\big) + v(t)r(t) \\ 
        \frac{1}{m}\big(F_\text{yf}(t)+F_\text{yr}(t)\big)-u(t)r(t)\\
        \frac{1}{I_\text{zz}} \big(l_\text{f} F_\text{yf}(t) - l_\text{r} F_\text{yr}(t)\big)
    \end{bmatrix}
\end{equation}
where $l_\text{f}$ and $l_\text{r}$ are the distances from center of mass to the front and back of the vehicle, $I_\text{zz}$ is the vehicle's moment of inertia, and $m$ is the vehicle's mass.
Note: $l_\text{f}$, $l_\text{r}$, $I_\text{zz}$ and $m$ are all constants and are assumed to be known.
The tire forces along the longitudinal and lateral directions of the vehicle at time $t$ are $F_\text{xi}(t)$ and $F_\text{yi}(t)$ respectively, where the `i' subscript can be replaced by `f' for the front wheels or `r' for the rear wheels.


To describe the tire forces along the longitudinal and lateral directions, we first define the \emph{wheel slip ratio} as 
\begin{equation}
\label{eq:wheel_slip_ratio}
    \lambda_\text{i}(t) = \begin{dcases}
    \frac{r_\text{w}\whlspd(t)-u(t)}{u(t)} ~\quad \text{during braking}\\
    \frac{r_\text{w}\whlspd(t)-u(t)}{r_\text{w}\whlspd(t)} ~\quad \text{during acceleration}
    \end{dcases}
\end{equation}
where the `i' subscript can be replaced as described above by 'f' for the front wheels or 'r' for the rear wheels, $r_\text{w}$ is the wheel radius, $\whlspd(t)$ is the tire-rotational speed at time $t$, braking corresponds to whenever $r_\text{w}\whlspd(t)-u(t) < 0$, and acceleration corresponds to whenever $r_\text{w}\whlspd(t)-u(t) \geq 0$. 
Then the longitudinal tire forces \cite[Chapter 4]{ulsoy2012automotive} are computed as
\begin{align}
    F_\text{xf}(t) &=  \frac{mgl_\text{r}}{l}\mu(\lambda_\text{f}(t)), \label{eq: Fxf def}\\
    F_\text{xr}(t) &=  \frac{mgl_\text{f}}{l}\mu(\lambda_\text{r}(t)), \label{eq: Fxr def}
\end{align}
where $g$ is the gravitational acceleration constant, $l = l_\text{f}+l_\text{r}$, and $\mu(\lambda_\text{i}(t))$ gives the surface-adhesion coefficient and is a function of the surface being driven on \cite[Chapter 13.1]{ulsoy2012automotive}.
Note that in FWD vehicles, the longitudinal rear wheel tire force has a much simpler expression:
\begin{rem}[\cite{jazar2008vehicle}]
\label{rem: FWD Fxr}
In a FWD vehicle, $F_\text{xr}(t) = 0$ for all $t$.
\end{rem}

For the lateral direction, define \emph{slip angles} of front and rear tires as 
\begin{align}
    \alpha_\text{f}(t) &= \delta(t) - \frac{v(t)+l_\text{f}r(t)}{u(t)}, \label{eq: alpha_f def}\\
    \alpha_\text{r}(t) &= - \frac{v(t)-l_\text{r}r(t)}{u(t)}, \label{eq: alpha_r def}  
\end{align}
then the lateral tire forces \cite[Chapter 4]{ulsoy2012automotive} are real-valued functions of the slip angles:
\begin{align}
    F_\text{yf}(t) &= c_{\alpha\text{f}}(\alpha_\text{f}(t)), \label{eq: Fyf def}\\
    F_\text{yr}(t) &= c_{\alpha\text{r}}(\alpha_\text{r}(t)). \label{eq: Fyr def}
\end{align}

Note $\mu,~c_{\alpha\text{f}}$ and $c_{\alpha\text{r}}$ are all nonlinear functions, but share similar characteristics. 
In particular, they behave linearly when the slip ratio and slip angle are close to zero, but saturate when the magnitudes of the slip ratio and slip angle reach some critical values of $\lambdac$ and $\alphac$, respectively, then decrease slowly \cite[Chapter 4, Chapter 13]{ulsoy2012automotive}.
As we describe in Section \ref{sec: satisfaction of linear regime}, during trajectory optimization we are able to guarantee that $\mu,~c_{\alpha\text{f}}$ and $c_{\alpha\text{r}}$ operate in the linear regime. 
As a result, to simplify exposition until we reach Section \ref{sec: satisfaction of linear regime}, we make the following assumption:
\begin{assum}
\label{ass: small slip} 
The absolute values of the the slip ratio and angle are bounded below their critical values (i.e., $|\lambda_\text{f}(t)|,|\lambda_\text{r}(t)|<\lambdac$ and $|\alpha_\text{f}(t)|, |\alpha_\text{r}(t)|<\alphac$ hold for all time).
\end{assum}
\noindent Assumption \ref{ass: small slip} ensures that the longitudinal tire forces can be described as
\begin{equation}
\label{eq: linear long tire force}
\begin{aligned}
F_\text{xf}(t) &=  \frac{mgl_\text{r}}{l}\bar \mu\lambda_\text{f}(t), \\
F_\text{xr}(t) &=  \frac{mgl_\text{f}}{l}\bar \mu\lambda_\text{r}(t),
\end{aligned}
\end{equation}
and the lateral tire forces can be described as
\begin{equation}
\label{eq: linear lat tire force}
\begin{aligned}
    F_\text{yf}(t) &= \bar c_{\alpha\text{f}}\alpha_\text{f}(t),\\
    F_\text{yr}(t) &= \bar c_{\alpha\text{r}}\alpha_\text{r}(t),
    \end{aligned}
\end{equation}
with constants $\bar\mu,\bar c_{\alpha\text{f}},\bar c_{\alpha\text{r}}\in\R$.
Note $\bar c_{\alpha\text{f}}$ and $\bar c_{\alpha\text{r}}$ are referred to as \emph{cornering stiffnesses}. 

Note that the steering angle of the front wheel, $\delta$, and the tire rotational speed, $\omega_\text{i}$, are the inputs that one is able to control. 
In particular for an AWD vehicle, both $\omega_\text{f}$ and $\omega_\text{r}$ are inputs; whereas, in a FWD vehicle only $\omega_\text{f}$ is an input.
When we formulate our controller in Section \ref{sec: controller} for a FWD vehicle we begin by assuming that we can directly control the front tire forces, $F_\text{xf}$ and $F_\text{yf}$. 
We then illustrate how to compute $\delta$ and $\omega_\text{f}$ when given $F_\text{xf}$ and $F_\text{yf}$.
For an AWD vehicle, we describe in Section \ref{sec: AWD}, how to compute $\delta$, $\omega_\text{f}$, and $\omega_\text{r}$.

In fact, as we describe in Section \ref{sec: controller}, our approach to perform control relies upon estimating the rear tire forces and controlling the front tire forces by applying appropriate tire speed and steering angle. 
Unfortunately, in the real-world our state estimation and models for front and rear tire forces may be inaccurate and aerodynamic-drag force could also affect vehicle dynamics \cite[Section 4.2]{ulsoy2012automotive}.
To account for the inaccuracy, we extend the vehicle dynamic model in \eqref{eq:highspeed_perfect} by introducing a time-varying affine modeling error $\Delta_u,\Delta_v,\Delta_r$ into the dynamics of $u,v,$ and $r$:
\begingroup
\renewcommand*{\arraystretch}{1.2} 
\begin{align}
    \dzhi(t) & = 
    \begin{bmatrix}
        u(t)\cos h(t) - v(t)\sin h(t)\\ u(t)\sin h(t) + v(t)\cos h(t) \\ r(t)\\
        \frac{1}{m}\big(F_\text{xf}(t)+F_\text{xr}(t)\big) + v(t)r(t) + \Delta_u(t) \\ 
        \frac{1}{m}\big(F_\text{yf}(t)+F_\text{yr}(t)\big)-u(t)r(t)+ \Delta_v(t)\\
        \frac{1}{I_\text{zz}} \big(l_\text{f} F_\text{yf}(t) - l_\text{r} F_\text{yr}(t)\big)+ \Delta_r(t)
    \end{bmatrix}. \label{eq:highspeed_noise}
\end{align}
\endgroup
Note, we have abused notation and redefined $\dzhi$ which was originally defined in \eqref{eq:highspeed_perfect}. 
For the remainder of this paper, we assume that the dynamics $\dzhi$ evolves according to \eqref{eq:highspeed_noise}.
To ensure that this definition is well-posed (i.e. their solution exists and is unique) and to aid in the development of our controller as described in Section \ref{sec: controller}, we make the following assumption:
\begin{assum}
\label{ass: dyn error bnd}
    $\Delta_u, \Delta_v,\Delta_r$ are all square integrable functions and are bounded (i.e., there exist real numbers $M_u, M_v, M_r\in [0,+\infty)$ such that $\|\Delta_u(t)\|_\infty\leq M_u, ~\|\Delta_v(t)\|_\infty\leq M_v, ~\|\Delta_r(t)\|_\infty \leq M_r$ for all $t$).
\end{assum}
\noindent Note in Section \ref{subsubsec: sysid}, we explain how to compute $\Delta_u,\Delta_v,\Delta_r$ using real-world data.

\subsection{Low-Speed Vehicle Model}

When the vehicle speed lowers below some critical value $\uc>0$, the denominator of the wheel slip ratio \eqref{eq:wheel_slip_ratio}  and tire slip angles \eqref{eq: alpha_f def} and \eqref{eq: alpha_r def} approach zero which makes applying the model described in \eqref{eq:highspeed_perfect} intractable. 
As a result, in this work when $u(t)\leq \uc$ the dynamics of a vehicle are modeled using a steady-state cornering model \cite[Chapter 6]{gillespie1992fundamentals}, \cite[Chapter 5]{balkwill2017performance}, \cite[Chapter 10]{dieter2018vehicle}.
Note that the critical velocity $\uc$ can be found according to \cite[(5) and (18)]{kim2019advanced}.

The steady-state cornering model or low-speed vehicle model is described using four states, $\zlo(t) = [x(t),y(t),h(t),u(t)]^\top\in\R^4$ at time $t$.
This model ignores transients on lateral velocity and yaw rate. 
Note that the dynamics of $x$, $y$, $h$ and $u$ are the same as in the high speed model \eqref{eq:highspeed_noise}; however, the steady-state corning model describes the yaw rate and lateral speed as
\begin{align}
    \vlo(t) =&l_\text{r}\rlo(t) - \frac{ml_\text{f}}{\bar c_{\alpha\text{r}}l}u(t)^2\rlo(t)   \label{eq: v lo} \\
    \rlo(t) =& \frac{\delta(t)u(t)}{l+C_\text{us} u(t)^2} \label{eq: r lo}
\end{align}
with understeer coefficient
\begin{equation}
    C_\text{us} = \frac{m}{l}\left( \frac{l_\text{r}}{\bar c_{\alpha \text{f}}} - \frac{l_\text{f}}{\bar c_{\alpha \text{r}}} \right).
\end{equation}
As a result, $\dzlo$ satisfies the dynamics of the first four states in \eqref{eq:highspeed_perfect} except with $\rlo$ taking the role of $r$ and $\vlo$ taking the role of $v$.

Notice when $u(t)=v(t)=r(t)=0$ and the longitudinal tire forces are zero, $\dot u(t)$ could still be nonzero due to a nonzero $\Delta_u(t)$.
To avoid this issue, we make a tighter assumption on $\Delta_u(t)$ without violating Assumption \ref{ass: dyn error bnd}:
\begin{assum}
\label{ass: dyn error bnd - low speed}
For all $t$ such that $u(t)\in[0,\uc]$, $|\Delta_u(t)|$ is bounded from above by a linear function of $u(t)$ (i.e.,
\begin{equation}
   |\Delta_u(t)| \leq \bpro\cdot u(t) + \boff, \text{ if } u(t)\in[0,\uc],
\end{equation}
where $\bpro$ and $\boff$ are constants satisfying $\bpro\cdot\uc+\boff\leq M_u$).
In addition, $\Delta_u(t)=0$ if $u(t) = 0$.
\end{assum}

As we describe in detail in Section \ref{subsec:hybrid_desc}, the high-speed and low-speed models can be combined together as a hybrid system to describe the behavior of the vehicle across all longitudinal speeds. 
In short, when $u$ transitions past the critical speed $\uc$ from above at time $t$, the low speed model's states are initialized as:
\begin{equation}
\label{eq: reset h2l}
    \zlo(t) = \pi_{1:4}(\zhi(t))
\end{equation}
where $\pi_{1:4}:\R^6\rightarrow\R^4$ is the projection operator that projects $\zhi(t)$ onto its first four dimensions via the identity relation.
If $u$ transitions past the critical speed from below at time $t$, the high speed model's states are initialized as
\begin{equation}
\label{eq: reset l2h}
    \zhi(t) = [\zlo(t)^\top, \vlo(t), \rlo(t)]^\top.
\end{equation}


\section{Trajectory Design and Safety}
\label{sec: trajectory design and safety}

This section describes the space of trajectories that are optimized over at run-time within \methodname{}, how this paper defines safety during motion planning via the notion of not-at-fault behavior, and what assumptions this paper makes about the environment surrounding the ego-vehicle. 

\subsection{Trajectory Parameterization}

Each trajectory plan is specified over a compact time interval. 
Without loss of generality, we let this compact time interval have a fixed duration $\tf$.
Because $\methodname{}$ performs receding-horizon planning, we make the following assumption about the time available to construct a new plan:
\begin{assum} \label{assum:tplan}
During each planning iteration starting from time $\tz$, the ego vehicle has $\tplan$ seconds to find a control input.
This control input is applied during the time interval $[\tz+\tplan, \tz+\tplan+\tf]$ where $\tf \geq 0$ is a user-specified constant. 
In addition, the state of the vehicle at time $\tz+\tplan$ is known at time $\tz$. 
\end{assum}

In each planning iteration, \methodname{} chooses a trajectory to be followed by the ego vehicle.
These trajectories are chosen from a pre-specified continuum of trajectories, with each uniquely determined by a \textit{trajectory parameter} $p \in \P$.
Let $\P \subset \R^{n_p}$, $n_p \in \N$ be a n-dimensional box $\Int(\underline p, \overline p)$ where $\underline p, \overline p\in\R^{n_p}$ indicate the element-wise lower and upper bounds of $p$, respectively.
We define these desired trajectories as follows:

\begin{defn}\label{def:traj_param}
For each $p \in \P$, a \emph{desired trajectory} is a function for the longitudinal speed, $\udes(\cdot,p):[\tz+\tplan,t_\text{f}] \to \R$, a function for the heading, $\hdes(\cdot,p):[\tz+\tplan,t_\text{f}] \to \R$, and a function for the yaw rate, $\rdes(\cdot,p): [\tz+\tplan,t_\text{f}] \to \R$, that satisfy the following properties.
\begin{enumerate}
\item For all $p\in\P$, there exists a time instant $\tnb\in[\tz+\tplan,\tf)$ after which the desired trajectory begins to brake (i.e., $|\udes(t,p)|$, $|\hdes(t,p)|$ and $|\rdes(t,p)|$ are non-increasing for all $t\in[\tnb,\tf]$).
\item The desired trajectory eventually comes to and remains stopped (i.e., there exists a $\tstop \in [\tz+\tplan,\tf]$ such that $\udes(t, p) = \hdes(t, p) = \rdes(t, p) = 0$ for all $t \geq \tstop$).
\item $\udes$ and $\hdes$ are piecewise continuously differentiable \cite[Chapter 6, $\S$1.1]{remmert1991theory} with respect to $t$ and $p$.
\item The time derivative of the heading function is equal to the yaw rate function (i.e., $\rdes(t,p)=\frac{\partial}{\partial t}\hdes(t,p)$ over all regions that $\hdes(t,p)$ is continuously differentiable with respect to $t$).
\end{enumerate}
\end{defn}
\noindent The first two properties ensure that a fail safe contingency braking maneuver is always available and the latter two properties ensure that the tracking controller described in Section \ref{sec: controller} is well-defined.
Note that sometimes we abuse notation and evaluate a desired trajectory for $t > \tf$.
In this instance, the value of the desired trajectory is equal to its value at $\tf$.


\subsection{Not-At-Fault}

In dynamic environments, avoiding collision may not always be possible (e.g. a parked car can be run into). 
As a result, we instead develop a trajectory synthesis technique which ensures that the ego vehicle is not-at-fault \cite{shalev2017formal}:
\begin{defn}\label{defn:notatfault}
The ego vehicle is \emph{not-at-fault} if it is stopped, or if it is never in collision with any obstacles while it is moving. 
\end{defn}
\noindent In other words, the ego vehicle is not responsible for a collision if it has stopped and another vehicle collides with it. 
One could use a variant of not-at-fault and require that when the ego-vehicle comes to a stop it leave enough time for all surrounding vehicles to come safely to a stop as well.
The remainder of the paper can be generalized to accommodate this variant of not-at-fault; however, in the interest of simplicity we use the aforementioned definition. 
\begin{rem}
\label{rem:notatfault}
Under Assumption \ref{ass: small slip}, neither longitudinal nor lateral tire forces saturate (i.e., drifting cannot occur).
As a result, if the ego vehicle has zero longitudinal speed, it also has zero lateral speed and yaw rate. 
Therefore in Definition \ref{defn:notatfault}, the ego vehicle being stopped is equivalent to its longitudinal speed being $0$.
\end{rem}

\subsection{Environment and Sensing}

To provide guarantees about vehicle behavior in a receding horizon planning framework and inspired by \cite[Section 3]{vaskov2019not}, we define the ego vehicle's footprint as:
\begin{defn}
\label{def: footprint}
Given $\W\subset\R^2$ as the world space, the ego vehicle is a rigid body that lies in a rectangular $\Oego:=\Int([-0.5L,-0.5W]^T,[0.5L,0.5W]^T )\subset\W$ with width $W>0$, length $L>0$ at time $t=0$. 
Such $\Oego$ is called the \emph{footprint} of the ego vehicle.
\end{defn}
In addition, we define the dynamic environment in which the ego vehicle is operating within as:
\begin{defn}
An \emph{obstacle} is a set $\mathcal O_i(t)\subset \W$ that the ego vehicle cannot intersect with at time $t$, where $i\in\I$ is the index of the obstacle and $\I$ contains finitely many elements. 
\end{defn}
\noindent The dependency on $t$ in the definition of an obstacle allows the obstacle to move as $t$ varies. 
However if the $i$-th obstacle is static, then $\mathcal O_i(t)$ remains constant at all time.
Assuming that the ego vehicle has a maximum speed $\vegomax$ and all obstacles have a maximum speed $\vobsmax$ for all time, we then make the following assumption on planning and sensing horizon. 
\begin{assum}
\label{ass: sense horizon}
The ego vehicle senses all obstacles within a sensor radius $S>(\tf +\tplan)\cdot(\vegomax+\vobsmax)+0.5\sqrt{L^2+W^2}$ around its center of mass.
\end{assum}
\noindent Assumption \ref{ass: sense horizon} ensures that any obstacle that can cause a collision between times $t\in[\tz+\tplan, \tz+\tplan+\tf]$ can be detected by the vehicle \cite[Theorem 15]{vaskov2019not}.
Note one could treat sensor occlusions as a obstacles that travel at the maximum obstacle speed \cite{yu2019occlusion,yu2020risk}.

\section{Controller Design and Hybrid System Model}
\label{sec: controller}

This section describes the control inputs that we use to follow the desired trajectories and describes the closed-loop hybrid system vehicle model. 
 Recall that the control inputs to the vehicle dynamics model are the steering angle of the front wheel, $\delta$, and the tire rotational speed, $\omega_\text{i}$. 
Section \ref{subsec: controller design} describes how to select front tire forces to follow a desired trajectory and Section \ref{subsec:inputs} describes how to compute a steering angle and tire rotational speed input from these computed front tire forces. 
Section \ref{subsec:hybrid_desc} describes the closed-loop hybrid system model of the vehicle under the chosen control input. 
Note that this section focuses on the FWD vehicle model. 


\subsection{Robust Controller}
\label{subsec: controller design}
Because applying reachability analysis to linear systems generates tighter approximations of the system behavior when compared to nonlinear systems,
we propose to develop a feedback controller that linearizes the dynamics. 
Unfortunately, because both the high-speed and low-speed models introduced in Section \ref{sec: dynamics} are under-actuated (i.e., the dimension of control inputs is smaller than that of system state), our controller is only able to partially feedback linearize the vehicle dynamics.
Such controller is also expected to be robust such that it can account for computational errors as described in Assumptions \ref{ass: dyn error bnd} and \ref{ass: dyn error bnd - low speed}.

We start by introducing the controller on longitudinal speed whose dynamics appears in both high-speed and low-speed models.
Recall $\|\Delta_u(t)\|_\infty\leq M_u$ in Assumption \ref{ass: dyn error bnd}.
Inspired by the controller developed in \cite{giusti2016ultimate}, we set the longitudinal front tire force to be
\begin{equation}
    \label{eq: Fxf}
    \begin{split}
    F_\text{xf}(t) = -m K_u (u(t) - \udes(t,p) ) + m\dudes(t,p) + \\ - F_\text{xr}(t)- mv(t)r(t) + m\tau_u(t,p),
    \end{split}
\end{equation}
where
\begin{align}
\tau_u(t,p) =& -\big(\kappa_u(t,p) M_u+\phi_u(t,p)\big)e_u(t,p), \label{eq: tau_u def} \\
    \kappa_u(t,p) =& \kappa_{1,u}+\kappa_{2,u}\int_{\tz}^t\|u(s)-\udes(s,p)\|^2ds,\label{eq: kappa_u def}\\
    \phi_u(t,p) =&\phi_{1,u}+\phi_{2,u}\int_{\tz}^t\|u(s)-\udes(s,p)\|^2ds,\label{eq: phi_u def}\\
    e_u(t,p) =& u(t)-\udes(t,p), \label{eq: e_u def}
\end{align}
with user-chosen constants $\kappa_{1,u},\kappa_{2,u},\phi_{1,u},\phi_{2,u}\in\R_+$.
Note in \eqref{eq: Fxf} we have suppressed the dependence on $p$ in $F_\text{xf}(t)$ for notational convenience.
Using \eqref{eq: Fxf}, the closed-loop dynamics of $u$ become:
\begin{equation}
\begin{split}
    \dot u(t) = \tau_u(t,p) + \Delta_u(t) + \dudes(t,p) + \\ -K_u \left( u(t) - \udes(t,p)\right).
    \end{split}
    \label{eq: u_dot_closed_loop_eq}
\end{equation}

The same control strategy can be applied to vehicle yaw rate whose dynamics only appear in  the high-speed vehicle model.
Let the lateral front tire force be 
\begin{equation}
\label{eq: Fyf}
\begin{split}
     F_\text{yf}(t) = -\frac{I_\text{zz} K_r}{l_\text{f}}\left( r(t) - \rdes(t,p) \right) +  \frac{I_\text{zz}}{l_\text{f}} \drdes(t,p) + \\ -\frac{I_\text{zz} K_h}{l_\text{f}}\left( h(t) - \hdes(t,p)\right) + \frac{l_\text{r}}{l_\text{f}} F_\text{yr}(t) + \frac{I_\text{zz}}{l_\text{f}} \tau_r(t,p),
     \end{split}
\end{equation}
where
\begin{align}
 \tau_r(t,p) =& -\big(\kappa_r(t,p) M_r+\phi_r(t,p)\big)e_r(t,p) \label{eq: tau_r def} \\
    \kappa_r(t,p) =& \kappa_{1,r}+\kappa_{2,r}\int_{\tz}^t\left\| \begin{bmatrix}r(s)\\h(s)\end{bmatrix} - \begin{bmatrix}\rdes(s,p)\\\hdes(s,p)\end{bmatrix}\right\|^2ds\\
    \phi_r(t,p) =&\phi_{1,r}+\phi_{2,r}\int_{\tz}^t\left\| \begin{bmatrix}r(s)\\h(s)\end{bmatrix} - \begin{bmatrix}\rdes(s,p)\\\hdes(s,p)\end{bmatrix}\right\|^2ds\\
    e_r(t,p) =& \begin{bmatrix} K_r & K_h \end{bmatrix} \begin{bmatrix} r(t)-\rdes(t,p) \\ h(t)-\hdes(t,p) \end{bmatrix}
\end{align}
with user-chosen constants $\kappa_{1,r},\kappa_{2,r},\phi_{1,r},\phi_{2,r}\in\R_+$.
Note in \eqref{eq: Fyf} we have again suppressed the dependence on $p$ in $F_\text{yf}(t)$ for notational convenience.
Using \eqref{eq: Fyf}, the closed-loop dynamics of $r$ become:
\begin{equation}
    \begin{split}
    \dot r(t) = &\tau_r(t,p)+ \Delta_r(t)+ \drdes(t,p) + \\ & -K_r\big(r(t) - \rdes(t,p)\big)+  \\  &-K_h\big(h(t) - \hdes(t,p)\big).
    \end{split}
    \label{eq: r_dot_closed_loop_eq}
\end{equation}
Using \eqref{eq: Fyf}, the closed-loop dynamics of $v$ become:
\begin{equation}
    \begin{split}
        \dot v(t) = \frac{1}{m}\Bigg(\frac{l}{l_\text{f}} F_\text{yr}(t) +\frac{I_\text{zz}}{l_\text{f}}\Big(\tau_r(t,p) + \drdes(t,p) + \\
        - u(t)r(t)+\Delta_v(t) -K_r\big(r(t) - \rdes(t,p)\big) + \\ -K_h\big(h(t) - \hdes(t,p)\big) \Big) \Bigg).
    \end{split}
    \label{eq: v_dot_closed_loop_eq}
\end{equation}

Because $\udes$, $\rdes$, and $\hdes$ depend on trajectory parameter $p$, one can rewrite the closed loop high-speed and low-speed vehicle models as
\begin{align}
    \dzhi(t) &= \fhi(t,\zhi(t),p), \label{eq: hi_veh_closeloop}\\
    \dzlo(t) &= \flo(t,\zlo(t),p), \label{eq: lo_veh_closeloop}
\end{align}
where dynamics of $x$, $y$ and $h$ are stated as the first three dimensions in \eqref{eq:highspeed_perfect}, closed-loop dynamics of $u$ is described in \eqref{eq: u_dot_closed_loop_eq}, and closed-loop dynamics of $v$ and $r$ in the high-speed model are presented in \eqref{eq: v_dot_closed_loop_eq} and \eqref{eq: r_dot_closed_loop_eq}.
Note that the lateral tire force could be defined to simplify the dynamics on $v$ instead of $r$, but the resulting closed loop system may differ.
Controlling the yaw rate may be easier in real applications, because $r$ can be directly measured by an IMU unit.

\subsection{Extracting Wheel Speed and Steering Inputs}\label{subsec:inputs}

Because we are unable to directly control tire forces, it is vital to compute wheel speed and steering angle such that the proposed controller described in \eqref{eq: Fxf} and \eqref{eq: Fyf} is viable.
Under Assumption \ref{ass: small slip}, wheel speed and steering inputs can be directly computed in closed form.
The wheel speed to realize longitudinal front tire force \eqref{eq: Fxf} can be derived from \eqref{eq:wheel_slip_ratio} and \eqref{eq: linear long tire force} as
\begin{equation}
\label{eq: compute w_cmd}
    \omega_\text{f}(t) = \begin{dcases}
    \left(\frac{lF_\text{xf}(t)}{\bar\mu mgl_\text{r}} + 1\right)\frac{u(t)}{r_\text{w}} ~\quad \text{during braking},\\
    \frac{u(t)}{\left(1-\frac{lF_\text{xf}(t)}{\bar\mu mgl_\text{r}}\right)r_\text{w}} \hspace{1.05cm} \text{during acceleration}.
    \end{dcases}
\end{equation}
Similarly according to \eqref{eq: alpha_f def} and \eqref{eq: linear lat tire force}, the steering input 
\begin{equation}
    \delta(t) = \frac{F_\text{yf}(t)}{\bar c_{\alpha\text{f}}} + \frac{v(t)+l_\text{f}r(t)}{u(t)}
\end{equation}
achieves the lateral front tire force in \eqref{eq: Fyf} when $u(t)>\uc$.

Notice lateral tire forces does not appear in the low-speed dynamics, but one is still able to control the lateral behavior of the ego vehicle.
Based on \eqref{eq: v lo} and \eqref{eq: r lo}, yaw rate during low-speed motion is directly controlled by steering input $\delta(t)$ and lateral velocity depends on yaw rate.
Thus to achieve desired behavior on the lateral direction, one can set the steering input to be
\begin{equation}
    \label{eq: low speed lateral control}
    \delta(t) = \frac{\rdes(t)(l+C_\text{us}u(t)^2)}{u(t)}.
\end{equation}

\subsection{Augmented State and Hybrid Vehicle Model}
\label{subsec:hybrid_desc}

To simplify the presentation throughout the remainder of the paper, we define a hybrid system model of the vehicle dynamics that switches between the high and low speed vehicle models when passing through the critical longitudinal velocity. 
In addition, for computational reasons that are described in subsequent sections, we augment the initial condition of the system to the state vector while describing the vehicle dynamics.
In particular, denote $z_0 = [(\zpos_0)^\top, (\zvel_0)^\top]^\top\in\Z_0\subset\R^6$ the initial condition of the ego vehicle where $\zpos_0=[x_0,y_0,h_0]^\top\in\R^3$ gives the value of $[x(t),y(t),h(t)]^\top$ and $\zvel_0=[u_0,v_0,r_0]^\top\in\R^3$ gives the value of $[u(t),v(t),r(t)]^\top$ at time $t=0$.
Then we augment the initial velocity condition $\zvel_0\in\R^3$ of the vehicle model and trajectory parameter $p$ into the vehicle state vector as $\zaug(t) = [x(t),y(t),h(t),u(t),v(t),r(t), (\zvel_0)^\top, p^\top]^\top\in\R^{9}\times\P\subset\R^{9 + n_p}$.
Note the last $3 + n_p$ states are static with respect to time.
As a result, the dynamics of the augmented vehicle state during high-speed and low-speed scenarios can be written as
\begin{equation}
\label{eq: dyn tilde_z}
    \dzaug(t) = \begin{dcases}
        \begin{bmatrix}
            \fhi(t,\zhi(t),p) \\ 0_{ (3 + n_p) \times1}
        \end{bmatrix} , \text{ if } u(t) > \uc, \\ \\
        \begin{bmatrix}
            \flo(t,\zlo(t),p) \\ 0_{(5 + n_p)\times1}
        \end{bmatrix} , \text{ if } u(t) \leq \uc,
    \end{dcases}
\end{equation}
which we refer to as the \emph{hybrid vehicle dynamics model}.
Notice when $u(t)\leq \uc$, assigning zero dynamics to $v$ and $r$ in \eqref{eq: dyn tilde_z} does not affect the evolution of the vehicle's dynamics because the lateral speed and yaw rate are directly computed via longitudinal speed as in \eqref{eq: v lo} and  \eqref{eq: r lo}.

Because the vehicle's dynamics changes depending on $u$, it is natural to model the ego vehicle as a hybrid system $HS$ \cite[Section 1.2]{lunze2009handbook}. 
The hybrid system has $\zaug$ as its state and consists of a high-speed mode and a low-speed mode with dynamics in \eqref{eq: dyn tilde_z}.
Instantaneous transition between the high and low speed models within $HS$ are described using the notion of a \emph{guard} and \emph{reset map}.
The guard triggers a transition and is defined as $\{\zaug(t)\in\R^{9}\times\P\mid u(t) =\uc\}$. 
Once a transition happens, the reset map maintains the last $3+n_p$ dimensions of $\zaug(t)$, but resets the first $6$ dimensions of $\zaug(t)$ via \eqref{eq: reset h2l} if $u(t)$ approaches $\uc$ from above and via \eqref{eq: reset l2h} if $u(t)$ approaches $\uc$ from below. 

We next prove that for desired trajectory defined as in Definition \ref{def:traj_param} under the controllers defined in Section \ref{sec: controller}, the vehicle model eventually comes to a stop. 
To begin note that experimentally, we observed that the vehicle quickly comes to a stop during braking once its longitudinal speed becomes $u(t) \leq 0.15$[m/s].
Thus we make the following assumption:
\begin{assum}
\label{ass: small speed can stop}
Suppose $u(t) = 0.15$ for some $t \geq \tstop$. 
Then under the control inputs \eqref{eq: Fxf} and \eqref{eq: Fyf} and while tracking any desired trajectory as in Definition \ref{def:traj_param}, the ego vehicle takes at most $\tfstop$ seconds after $\tstop$ to come to a complete stop.
\end{assum}

We use this assumption to prove that the vehicle can be brought to a stop within a specified amount of time in the following lemma whose proof can be found in Appendix \ref{app:braking}:
\begin{lem}\label{lem:braking}
Let $\Z_0 \subset \R^6$ be a compact subset of initial conditions for the vehicle dynamic model and $\P$ be a compact set of trajectory parameters. 
Let $\Delta_u(t)$ be bounded for all $t$ as in Assumptions \ref{ass: dyn error bnd} and \ref{ass: dyn error bnd - low speed} with constants $M_u$, $\bpro$ and $\boff$.
Let $\zaug$ be a solution to the hybrid vehicle dynamics model \eqref{eq: dyn tilde_z} beginning from $z_0 \in \Z_0$ under trajectory parameter $p \in \P$ while applying the control inputs \eqref{eq: Fxf} and \eqref{eq: Fyf} to track some desired trajectory satisfying Definition \ref{def:traj_param}.
Assume the desired longitudinal speed satisfies the following properties: $\udes(0,p) = u(0)$, $\udes(t,p)$ is only discontinuous at time $\tstop$, and $\udes(t,p)$ converges to $\uc$ as $t$ converges to $\tstop$ from below.
If $K_u$, $\kappa_{1,u}$ and $\phi_{1,u}$ are chosen such that $\frac{M_u}{\kappa_{1,u}M_u+\phi_{1,u}}\in(0.15,\uc]$ and $\frac{(\boff)^2}{4(\kappa_{1,u}M_u+\phi_{1,u}-\bpro)} < 0.15^2 K_u$ hold, then for all $p \in \P$ and $z_0 \in \Z_0$ satisfying $u(0)>0$, there exists $\tb$ such that $u(t)=0$ for all $t \geq \tb$. 
\end{lem}

Note, the proof of Lemma \ref{lem:braking} includes an explicit formula for $\tb$ in \eqref{eq: tb computation}. 
This lemma is crucial because it specifies the length of time over which we should construct FRS, so that we can verify that not-at-fault behavior can be satisfied based on Definition \ref{defn:notatfault} and Remark \ref{rem:notatfault}.

\section{Computing and Using the FRS}
\label{sec: rtd}

This section describes how \methodname{} operates at a high-level.
It then describes the offline reachability analysis of the ego vehicle as a state-augmented hybrid system using zonotopes and illustrates how the ego vehicle's footprint can be accounted for during reachability analysis.

\methodname{} conservatively approximates a control-parameterized FRS of the full-order vehicle dynamics.
The FRS includes all behaviors of the ego vehicle over a finite time horizon and is mathematically defined in Section \ref{subsec: FRS computation}.
To ensure the FRS is a tight representation, \methodname{} relies on the controller design described in Section \ref{sec: controller}.
Because this controller partially linearizes the dynamics, \methodname{} relies on a zonotope-based reachable set representation which behave well for nearly linear systems.


During online planning, \methodname{} performs control synthesis by solving optimization problems in a receding horizon fashion, where the optimization problem computes a trajectory parameter to navigate the ego vehicle to a waypoint while behaving in a not-at-fault manner. 
As in Assumption \ref{assum:tplan}, each planning iteration in \methodname{} is allotted $\tplan>0$ to generate a plan.
As depicted in Figure \ref{fig:timeline}, if a particular planning iteration begins at time $t_0$, its goal is to find a control policy by solving an online optimization within $\tplan$ seconds so that the control policy can be applied during $[t_0+\tplan, t_0+\tplan+ \tf]$.
Because any trajectory in Definition \ref{def:traj_param} brings the ego vehicle to a stop, we partition $[t_0+\tplan, t_0+\tplan+ \tf]$ into $[t_0+\tplan,t_0+\tplan+\tnb)$ during which a driving maneuver is tracked and $[t_0+\tplan+\tnb,t_0+\tplan+\tf]$ during which a contingency braking maneuver is activated.
Note $\tnb$ is not necessarily equal to $\tstop$.
As a result of Lemma \ref{lem:braking}, by setting $\tf$ equal to $\tb$ one can guarantee that the ego vehicle comes to a complete stop by $\tf$.

\begin{figure}[t]
    \centering
    \includegraphics[trim={0cm, 16.2cm, 11.7cm, 0cm},clip,width=\columnwidth]{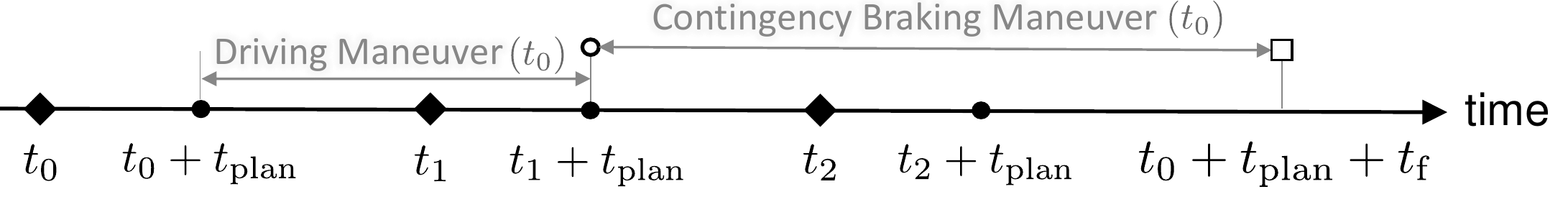}
    \caption{An illustration of 3 successive planning/control iterations.
             $\tplan$ seconds are allotted to compute a planned trajectory.
             Each plan is of duration $\tf$ and consists of a driving maneuver of duration $\tnb$ and a contingency braking maneuver.
             Diamonds denote the time instances where planning computations begin and $t_2-t_1=t_1-t_0=\tnb$.
             Filled-in circles denote the instances where feasible driving maneuvers are initiated.
             If the planning phase between $[t_1,t_1+\tplan]$ is infeasible, the contingency braking maneuver whose feasibility is verified during the planning phase between $[t_0, t_0+\tplan]$ is applied.}
    \label{fig:timeline}
    \vspace*{-.5cm}
\end{figure}
If the planning iteration at time $\tz$ is feasible (i.e., not-at-fault), then the entire feasible planned driving maneuver is applied during $[\tz+\tplan, \tz+\tplan+\tnb)$.
If the planning iteration starting at time $\tz$ is infeasible, then the braking maneuver, whose safe behavior was verified in the previous planning iteration, can be applied starting at $\tz+\tplan$ to bring the ego vehicle to a stop in a not-at-fault manner.
To ensure real-time performance, $\tplan\leq\tnb$.
To simplify  notation, we reset time to $0$ whenever a feasible control policy is about to be applied.

\subsection{Offline FRS Computation}
\label{subsec: FRS computation}

The \emph{FRS} of the ego vehicle is
\begin{equation}
 \label{eq: FRS def}
\begin{split}
    \FRS([0,&\tf])=\Bigg\{(x,y)\in \W\mid \exists t\in[0,\tf], p\in\P,  \\
    & z_0 = \begin{bmatrix} \zpos_0 \\ \zvel_0\end{bmatrix}\in\Z_0 \text{ s.t. }\begin{bmatrix} x\\ y \end{bmatrix} = \pi_{xy}(\zaug(t)),   \\
    & \zaug \text{ is a solution of } HS \text{ with } \zaug(0)= \begin{bmatrix} z_0 \\  \zvel_0 \\ p \end{bmatrix} \Bigg\},
\end{split}
 \end{equation}
where $\pi_{xy}:\R^{9+n_p}\rightarrow \R^2$ is the projection operator that outputs the first two coordinates from its argument.
$\FRS([0,\tf])$ collects all possible behavior of the ego vehicle while following the dynamics of $HS$ in the $xy$-plane over time interval $[0,\tf]$ for all possible $p\in\P$ and initial condition $z_0\in\Z_0$.
Computing $\FRS([0,\tf])$ precisely is numerically challenging because the ego vehicle is modeled as a hybrid system with nonlinear dynamics, thus we aim to compute an outer-approximation of $\FRS([0,\tf])$ instead.

To outer-approximate $\FRS([0,\tf])$, we start by making the following assumption:
\begin{assum}
\label{ass: initial 0 pos condition}
The initial condition space $\Z_0 = \{0_{3\times1}\}\times \Zvel_0$ where $\Zvel_0=\text{int}(\underline{\zvel_0}, \overline{\zvel_0})\subset\R^3$ is a 3-dimensional box representing all possible initial velocity conditions $\zvel_0$ of the ego vehicle.
\end{assum}
\noindent Because vehicles operate within a bounded range of speeds, this assumption is trivial to satisfy.
Notice in particular that $\Zvel_0$ is a zonotope $\zonocg{\cvel_0}{\Gvel_0}$ where $\cvel_0 = \frac{1}{2}(\underline{\zvel_0}+\overline{\zvel_0})$ and $\Gvel_0 = \frac{1}{2}\diag(\overline{\zvel_0}-\underline{\zvel_0})$.
We assume a zero initial position condition $\zpos_0$ in the first three dimensions of $\Z_0$ for simplicity, and nonzero $\zpos_0$ can be dealt with via coordinate transformation online as described in Section \ref{subsec: nonzero pos}.

Recall that because $\P$ is a compact n-dimensional box, it can also be represented as a zonotope as $\zonocg{c_p}{G_p}$ where $c_p = \frac{1}{2}(\underline p+\overline p)$ and $G_p = \frac{1}{2}\diag(\overline p - \underline p)$. 
Then the set of  initial conditions for $\zaug(0)$ can be represented as a zonotope $\Zaug_0=\zonocg{c_{\zaug}}{G_{\zaug}} \subset \Z_0\times\Zvel_0\times\P$ where 
\begin{equation}
    c_{\zaug} = \begin{bmatrix}
    0_{3\times 1}\\ \cvel_0\\ \cvel_0\\ c_p
    \end{bmatrix}, ~ G_{\zaug} = \begin{bmatrix}
    0_{3\times3} & 0_{3\times n_p}\\ \Gvel_0 & 0_{3\times n_p}\\ \Gvel_0 & 0_{3\times n_p} \\ 0_{n_p \times3} & G_p
    \end{bmatrix}.
\end{equation}
Observe that by construction each row of $G_{\zaug}$ has at most $1$ nonzero element. 
Without loss of generality, we assume $\Gvel_0$ and $G_p$ has no zero rows.
If there was a zero row it would mean that the corresponding dimension can only take one value and does not need to be augmented or traced in $\zaug$ for reachability analysis.

Next we pick a time step $\Delta_t\in\R_+$ such that $\tf/\Delta_t\in\N$, and partition the time interval $[0,\tf]$ into $\tf/\Delta_t$ \emph{time segments} as $T_j=[(j-1)\Delta_t, j\Delta_t]$ for each $j\in \J=\{1,2,\cdots,\tf/\Delta_t\}$.
Finally we use an open-source toolbox CORA \cite{althoff2015introduction}, which takes $HS$ and the initial condition space $\Zaug_0$, to over-approximate the FRS in \eqref{eq: FRS def} by a collection of zonotopes $\{\RR_j\}_{j\in\J}$ over all time intervals where $\RR_j\subset\R^{9+n_p}$.
As a direct application of Theorem 3.3, Proposition 3.7 and the derivation in Section 3.5.3 in \cite{althoff2010reachability}, one can conclude the following theorem:
\begin{thm}
\label{thm: FRS over-approximation}
Let $\RR_j\subset\R^{9+n_p}$ be the zonotopes computed by CORA under the hybrid vehicle dynamics model beginning from $\Zaug_0$.
Let $\zaug$ be a solution to hybrid system $HS$ starting from an initial condition in $\Zaug_0$.
Then $\zaug(t) \in\RR_j$ for all $j\in\J$ and $t\in T_j$ and
\begin{equation}
\label{eq: FRS inside Zono}
   \FRS([0,\tf]) \subset \bigcup_{j\in\J}\pi_{xy}(\RR_j).
\end{equation}
\end{thm}
\noindent Notice in \eqref{eq: FRS inside Zono} we have abused notation by extending the domain of  $\pi_{xy}$ to any zonotope $\Z = \zonocg{c}{G}$ in $\R^{9+n_p}$ as
\begin{equation}
   \pi_{xy}(\Z) = \left<\begin{bmatrix}
    [c]_1 \\ [c]_2
    \end{bmatrix},~\begin{bmatrix}
    [G]_{1:} \\ [G]_{2:}
    \end{bmatrix}\right>.
\end{equation}

\subsection{Slicing}

The FRS computed in the previous subsection contain the behavior of the hybrid vehicle dynamics model for all initial conditions belonging to $\Z_0$ and $\P$. 
To use this set during online optimization, \methodname{} plugs in the predicted initial velocity of the vehicle dynamics at time $\tz + \tplan$ and then optimizes over the space of trajectory parameters. 
Recall the hybrid vehicle model is assumed to have zero initial position condition during the computation of $\{\RR_j\}_{j\in\J}$ by Assumption \ref{ass: initial 0 pos condition}.
This subsection describes how to plug in the initial velocity into the pre-computed FRS. 

We start by describing the following useful property of the zonotopes $\RR_j$ that make up the FRS, which follows from Lemma 22 in \cite{holmes2020reachable}:
\begin{prop}
\label{prop: sliceable}
    Let $\{\RR_j = \zonocg{c_{\RR_j}}{G_{\RR_j}}\}_{j \in \J}$ be the set of zonotopes computed by CORA under the hybrid vehicle dynamics model beginning from $\Zaug_0$.
    Then for any $j\in\J$, $G_{\RR_j}=[g_{\RR_j,1},g_{\RR_j,2},\ldots,g_{\RR_j,\ell_j}]$ has only one generator, $g_{\RR_j,b_k}$, that has a nonzero element in the $k$-th dimension for each $k \in \{7, \ldots, (9 + n_p) \}$.
    In particular, $b_{k}\neq b_{k'}$ for $k \neq k'$. 
\end{prop}
\noindent We refer to the generators with a nonzero element in the $k$-th dimension for each $k \in \{7, \ldots, (9 + n_p) \}$ as a \emph{sliceable generator} of $\RR_j$ in the $k$-th dimension.
In other words, for each $\RR_j= \zonocg{c_{\RR_j}}{G_{\RR_j}}$, there are exactly $3 + n_p$ nonzero elements in the last  $3 + n_p$ rows of $G_{\RR_j}$, and none of these nonzero elements appear in the same row or column. 
By construction $\Zaug_0$ has exactly $3 + n_p$ generators, which are each sliceable.
Using Proposition \ref{prop: sliceable}, one can conclude that $\RR_j$ has no less than $3 + n_p$ generators (i.e., $\ell\geq 3 + n_p$). 

Proposition \ref{prop: sliceable} is useful because it allows us to take a known $\zvel_0 \in \Zvel_0$ and $p \in \P$ and plug them into the computed $\{\RR_j\}_{j\in\J}$ to generate a \emph{slice} of the conservative approximation of the FRS that includes the evolution of the hybrid vehicle dynamics model beginning from $\zvel_0$ under trajectory parameter $p$.
In particular, one can plug the initial velocity into the sliceable generators as described in the following definition:

\begin{defn} \label{defn:slice}
Let $\{\RR_j = \zonocg{c_{\RR_j}}{G_{\RR_j}}\}_{j \in \J}$ be the set of zonotopes computed by CORA under the hybrid vehicle dynamics model beginning from $\Zaug_0$ where $G_{\RR_j}=[g_{\RR_j,1},g_{\RR_j,2},\ldots,  g_{\RR_j,{\ell_j}}]$.
Without loss of generality, assume that the sliceable generators of each $\RR_j$ are the first $3+n_p$ columns of $G_{\RR_j}$.
In addition, without loss of generality assume that the sliceable generators are ordered so that the dimension in which the non-zero element appears is increasing. 
The \emph{slicing operator} $\slice: P(\R^{9 + n_p})\times\Zvel_0\times\P\rightarrow  P(\R^{9 + n_p})$ is defined as
\begin{equation}
\label{eq: slice def}
   \hspace*{-0.2cm} \slice( \RR_j,\zvel_0,p) = \zonocg{c^\text{slc}}{[g_{\RR_j,(4 +n_p)},\ldots, g_{\RR_j,{\ell_j}}]}
\end{equation}
where 
\begin{equation}
\label{eq: slice center}
\begin{split}
    c^\text{slc} = c_{\RR_j} + \sum_{k=7}^{9}&\frac{[\zvel_0]_{(k-6)}-[c_{\RR_j}]_k}{[g_{\RR_j,(k-6)}]_k}g_{\RR_j,(k-6)} + \\ &+\sum_{k=10}^{9+n_p}\frac{[p]_{(k-9)}-[c_{\RR_j}]_k}{[g_{\RR_j,(k-6)}]_k}g_{\RR_j,(k-6)}.
\end{split}
\end{equation}
\end{defn}
\noindent Note, that in the interest of avoiding introducing novel notation, we have abused notation and assumed that the domain of $\slice$ is $P(\R^{9 + n_p})$ rather than the space of zonotopes in $P(\R^{9 + n_p})$. 
However, throughout this paper we only plug in zonotopes belonging to $P(\R^{9 + n_p})$ into the first argument of $\slice$.
Using this definition, one can show the following useful property whose proof can be found in Appendix \ref{app:slicing}:
\begin{thm}
\label{thm: slice_is_good}
Let $\{\RR_j \}_{j \in \J}$ be the set of zonotopes computed by CORA under the hybrid vehicle dynamics model beginning from $\Zaug_0$ and satisfy the statement of Definition \ref{defn:slice}.
    Then for any $j\in\J$, $z_0=[0,0,0,(\zvel_0)^\top]^\top\in\Z_0$, and $p\in\P$,  $\slice( \RR_j,\zvel_0,p)\subset \RR_j$.
    In addition, suppose $\zaug$ is a solution to $HS$ with initial condition $z_0$ and control parameter $p$.
    Then for each $j \in \J$ and $t \in T_j$
    \begin{equation}
        \zaug(t)\in \slice( \RR_j,\zvel_0,p ).
    \end{equation}
\end{thm}

\subsection{Accounting for the Vehicle Footprint in the FRS}

The conservative representation of the FRS generated by CORA only accounts for the ego vehicle's center of mass because $HS$ treats the ego vehicle as a point mass.
To ensure not-at-fault behavior while planning using \methodname{}, one must account for the footprint of the ego vehicle, $\Oego$, as in Definition \ref{def: footprint}.

To do this, define a projection operator $\pi_h: \{R_j\}_{j\in\J}\rightarrow P(\R)$ as $\pi_h(R_j) \mapsto \zonocg{[c_{R_j}]_3}{[G_{R_j}]_{3:}}$ where $R_j=\zonocg{c_{R_j}}{G_{R_j}}$ is a zonotope computed by CORA as described in Section \ref{subsec: FRS computation}.
Then by definition $\pi_h(\RR_j)$ is a zonotope and it conservatively approximates of the ego vehicle's heading during $T_j$.
Moreover, because $\pi_h(\RR_j)$ is a 1-dimensional zonotope, it can be rewritten as a 1-dimensional box $\Int(\hmid-\hrad, \hmid+\hrad)$ where $\hmid = [c_{\RR_j}]_3$ and $\hrad = \Sum(|[G_{\RR_j}]_{3:}|)$.
We can then use $\pi_h$ to define a map to account for vehicle footprint within the FRS:
\begin{defn} \label{defn:footprint}
Let $\RR_j$ be the zonotope computed by CORA under the hybrid vehicle dynamics model beginning from $\Zaug_0$ for arbitrary $j\in\J$, and denote $\pi_h(\RR_j)$ as $\Int(\hmid-\hrad, \hmid+\hrad)$.
Let $\SS \subset \W$ be a 2-dimensional box centered at the origin with length $\sqrt{L^2+W^2}$ and width $\frac{1}{2}L|\sin(\hrad)|+\frac{1}{2}W|\cos(\hrad)|$.
Define the \emph{rotation map} $\rot:P(\R)\rightarrow P(\W)$ as
\begin{equation}
\label{eq: def rot}
    \rot\big(\pi_h(\RR_j)\big) = \begin{bmatrix}
        \cos(\hmid)& -\sin(\hmid)\\\sin(\hmid) & \cos(\hmid)
    \end{bmatrix} \SS.
\end{equation}
\end{defn}
\noindent Note that in the interest of simplicity, we have abused notation and assumed that the argument to $\rot$ is any subset of $\R$.
In fact, it must always be a $1$-dimensional box. 
In addition note that $\rot\big(\pi_h(\RR_j)\big)$ is a zonotope because  the 2-dimenional box $\SS$ is equivalent to a 2-dimensional zonotope and it is multiplied by a matrix via \eqref{eq: zono-matrix mult}.
By applying geometry, one can verify that by definition $\SS$ bounds the area that $\Oego=\Int([-0.5L, -0.5W]^\top, [0.5L, 0.5W]^\top)$ travels through while rotating within the range $[-\hrad, \hrad]$.
As a result, $\rot\big(\pi_h(\RR_j)\big)$ over-approximates the area over which $\Oego$ sweeps according to $\pi_h(\RR_j)$ as shown in Fig. \ref{fig: ego rotate}.

Because $\SS$ can be represented as a zonotope with 2 generators, one can denote $\rot(\pi_h(\RR_j))$ as  $\zonocg{c_\rot}{G_\rot}\subset\R^2$ where $G_\rot\in\R^{2\times2}$.
Notice $\rot(\pi_h(\RR_j))$ in \eqref{eq: def rot} is a set in $\W$ rather than the higher dimensional space where $\RR_j$ exists.
We extend $\rot(\pi_h(\RR_j))$ to $\R^{9+n_p}$ as
\begin{equation}
    \ROT(\pi_h(\RR_j)) := \left< \begin{bmatrix} c_\rot \\ 0_{(7+n_p)\times 1}\end{bmatrix},~
    \begin{bmatrix} G_\rot \\ 0_{(7+n_p)\times 2}\end{bmatrix}\right>.
\end{equation}
Using this definition, one can extend the FRS to account for the vehicle footprint as in the following lemma whose proof can be found in Appendix \ref{app:footprint}:

\begin{figure}[t]
    \centering
    \includegraphics[trim={0cm, 1.5cm, 18cm, 6.5cm}, clip,width=0.7\columnwidth,clip=true]{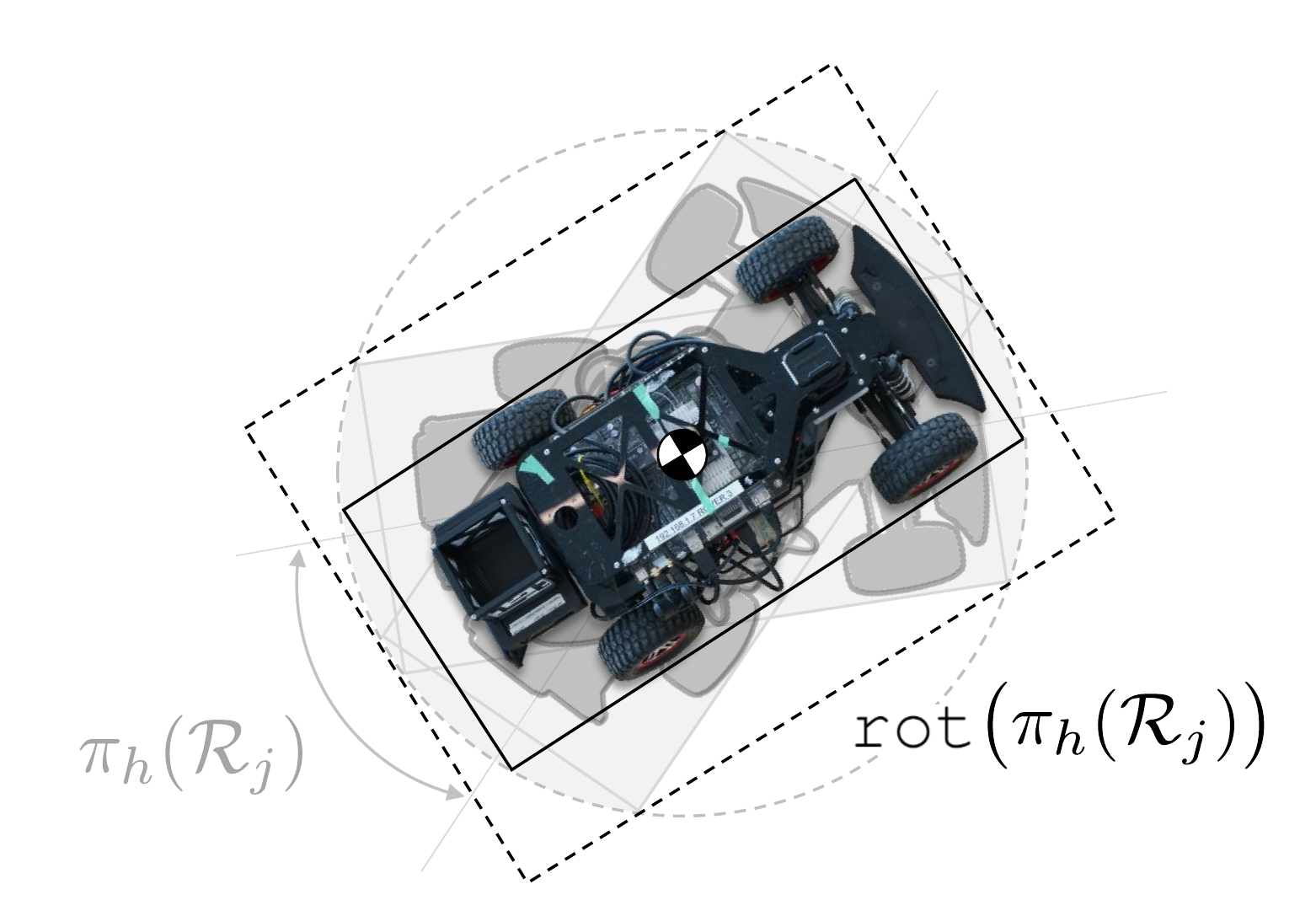}
    \caption{Rotation of the ego vehicle and its footprint within range $\pi_h(\RR_j)$. The ego vehicle with heading equals to the mean value of $\pi_h(\RR_j)$ is bounded by the box with solid black boundaries.
    The range of rotated heading is indicated by the grey arc.
    The area the ego vehicle's footprint sweeps is colored in grey, and is bounded by box $\rot\big(\pi_h(\RR_j)\big)$ with dashed black boundaries.}
    \label{fig: ego rotate}
\end{figure}

\begin{lem} \label{lem:footprint}
Let $\{\RR_j\}_{j \in \J}$ be the set of zonotopes computed by CORA under the hybrid vehicle dynamics model beginning from $\Zaug_0$.
Let $\zaug$ be a solution to $HS$ with initial velocity $\zvel_0$ and control parameter $p$ and 
let $\xi:P(\R^{9 + n_p})\times \Zvel_0\times\P\rightarrow P(\W)$ be defined as
\begin{equation}
\label{eq: xi def}
   \hspace*{-0.35cm} \xi(\RR_j,\zvel_0,p) = \pi_{xy}\Big(\slice\big(\RR_j\oplus\ROT(\pi_h(\RR_j)), \zvel_0, p\big)\Big).
\end{equation}
Then $\xi(\RR_j,\zvel_0,p)$ is a zonotope and for all $j \in \J$ and $t \in T_j$, the vehicle footprint oriented and  centered according to $\zaug(t)$ is contained within $\xi(\RR_j,\zvel_0,p)$.
\end{lem}


\noindent Again note that in the interest of simplicity we have abused notation and assumed that the first argument to $\xi$ is any subset of $\R^{9+n_p}$.
This argument is always a zonotope in  $\R^{9+n_p}$.


\section{Online Planning}
\label{sec: online}
This section begins by taking nonzero initial position condition into account and formulating the optimization for online planning in \methodname{} to search for a safety guaranteed control policy in real time.
It then explains how to represent each of the constraints of the online optimization problem in a differentiable fashion, and concludes by describing the performance of the online planning loop.

Before continuing we make an assumption regarding predictions of surrounding obstacles.
Because prediction is not the primary emphasis of this work, we assume that the future position of any sensed obstacle within the sensor horizon during $[\tz, \tz+\tplan+\tf]$ is conservatively known at time $t_0$:
\begin{assum}
\label{ass: obs in T}
    There exists a map $\vartheta:\J\times\I\rightarrow P(\W)$ such that $\vartheta(j,i)$ is a zonotope and
    \begin{equation}
        \cup_{t\in T_j}\OO_i(t) \cap \mathcal{B}\left( (x(t_0),y(t_0)), S \right) \subseteq \vartheta(j,i).
    \end{equation}
\end{assum}



\subsection{Nonzero Initial Position}
\label{subsec: nonzero pos}
Recall that the FRS computed in Section \ref{sec: rtd} is computed offline while assuming that the initial position of the ego vehicle is zero (i.e., Assumption \ref{ass: initial 0 pos condition}).
The zonotope collection $\{\RR_j\}_{j\in\J}$ can be understood as a local representation of the FRS in the local frame.
This local frame is oriented at the ego vehicle's location $[x_0,y_0]^\top\in\R^2$ with its $x$-axis aligned according to the ego vehicle's heading $h_0\in\R$, where $\zpos_0 = [x_0,y_0,h_0]^\top$ gives the ego vehicle's position $[x(t),y(t),h(t)]^\top$ at time $t=0$ in the world frame.
Similarly, $\xi(\RR_j,\zvel_0,p)$ is a local representation of the area that the ego vehicle may occupy during $T_j$ in the same local frame.

Because obstacles are defined in the world frame, to generate not-at-fault trajectories, one has to either transfer $\xi(\RR_j,\zvel_0,p)$ from the local frame to the world frame, or transfer the obstacle position $\vartheta(j,i)$ from the world frame to the local frame using a 2D rigid body transformation.
This work utilizes the second option and transforms $\vartheta(j,i)$ into the local frame as
\begin{equation}
\label{eq: vartheta local}
    \hspace*{-0.35cm}\vartheta^\text{loc}(j,i,\zpos_0)= \begin{bmatrix}
    \cos(h_0) & \sin(h_0) \\ -\sin(h_0) & \cos(h_0)
    \end{bmatrix}(\vartheta(j,i)-\begin{bmatrix} x_0\\y_0 \end{bmatrix}).
\end{equation}

\subsection{Online Optimization}
\label{subsec: online opt}

Given the predicted initial condition of the vehicle at $t = 0$ as $z_0 = [(\zpos_0)^\top, (\zvel_0)^\top]^\top\in\R^3\times\Zvel_0$, \methodname{} computes a not-at-fault trajectory by solving the following optimization problem at each planning iteration:
\begin{align*}
    \min_{p \in \P} & \quad \cost(z_0,p) \hspace{4cm} \opt\\
    \text{s.t.}
    & \quad \xi(\RR_j, \zvel_0,p) \cap \vartheta^\text{loc} (j,i,\zpos_0 )=\emptyset, \hspace{0.5cm} \forall j\in\J, \forall i\in\I
\end{align*}
where $\cost:\R^3\times\Zvel_0\times\P \to \R$ is a user-specified cost function and $\xi$ is defined as in Lemma \ref{lem:footprint}.
Note that the constraint in \opt is satisfied if for a particular trajectory parameter $p$, there is no intersection between any obstacle and the reachable set of the ego vehicle with its footprint considered during any time interval while following $p$.

\subsection{Representing the Constraint and its Gradient in \opt}
\label{subsec: constraint}

The following theorem, whose proof can be found in Appendix \ref{app:constraint}, describes how to represent the set intersection constraint in \opt and how to compute its derivative with respect to $p \in \P$:



\begin{thm}\label{thm:constraint}
  There exists matrices $A$ and $B$ and a vector $b$ such that $\xi(\RR_j,\zvel_0,p)\cap\vartheta^\text{loc}(j,i,\zpos_0)=\emptyset$ if and only if $\max(BA\cdot p - b) > 0$.
  In addition, the subgradient of $\max(BA\cdot p - b)$ with respect to $p$ is $\max_{k \in \hat{K}} [BA]_{k:}$,
 where $\hat{K} = \{k \mid [BA\cdot p - b]_k = \max(BA\cdot p - b)\}$.
\end{thm}
\noindent Formulas for the matrices $A$ and $B$ and vector $b$ in the previous theorem can be found in \eqref{eq: A def}, \eqref{eq:Bdef}, and \eqref{eq:bdef}, respectively.

\subsection{Online Operation}



 Algorithm \ref{alg:methodname} summarizes the online operations of  \methodname{}.
In each planning iteration, the ego vehicle executes the feasible control parameter that is computed in the previous planning iteration (Line 3).
Meanwhile, \texttt{SenseObstacles} senses and predicts obstacles as in Assumption \ref{ass: obs in T} (Line 4) in local frame decided by $\zpos_0$.
\opt is then solved to compute a control parameter $p^*$ using $z_0$ and $\{\vartheta^\text{loc}(j,i,\zpos_0)\}_{(j,i)\in\J\times\I}$ (Line 5). 
If \opt fails to find a feasible solution within $\tplan$, the contingency braking maneuver whose safety is verified in the last planning iteration is executed, and \methodname{} is terminated (Line 6).
In the case when \opt is able to find a feasible $p^*$, \texttt{StatePrediction} predicts the state value at $t=\tnb$ based on $z_0$ and $p^*$ as in Assumption \ref{assum:tplan} (Lines 7 and 8).
If the predicted velocity value does not belong to $\Zvel_0$, then its corresponding FRS is not available and the planning has to stop while executing a braking maneuver (Line 9).
Otherwise we reset time to 0 (Line 10) and start the next planning iteration.
Note Lines 4 and 7 are assumed to execute instantaneously, but in practice the time spent for these steps can be subtracted from $\tplan$ to ensure real-time performance.
By iteratively applying Definition \ref{defn:notatfault}, Lemmas \ref{lem:braking} and \ref{lem:footprint}, Assumption \ref{ass: obs in T} and \eqref{eq: vartheta local}, the following theorem holds:
\begin{thm}
Suppose the ego vehicle can sense and predict surrounding obstacles as in Assumption \ref{ass: obs in T}, and starts with a not-at-fault control parameter $p_0\in\P$.
Then by performing planning and control as in Algorithm \ref{alg:methodname}, the ego vehicle is not-at-fault for all time.
\end{thm}

\begin{algorithm}[t]
    \caption{\methodname{} Online Planning}
    \label{alg:methodname}
    \begin{algorithmic}[1]
        \REQUIRE $p_0\in\P$ and $z_0 = [(\zpos_0)^\top,(\zvel_0)^\top]^\top\in\R^3\times\Zvel_0$
        \STATE \textbf{Initialize:} $p^*=p_0$, $t=0$
        \STATE \textbf{Loop:} // {\it Line 3 executes at the same time as Line 4-8}
            \STATE \quad \textbf{Execute} $p^*$ during $[0, \tnb)$
            \STATE \quad  $\{\vartheta^\text{loc}(j,i,\zpos_0)\}_{(j,i)\in\J\times\I}\gets\texttt{SenseObstacles}()$
            \STATE \quad \textbf{Try} $p^*\gets\texttt{OnlineOpt}(z_0,\{\vartheta^\text{loc}(j,i,\zpos_0)\}_{(j,i)\in\J\times\I})$ \\ \quad // {\it within $\tplan$ seconds}
            \STATE \quad  \textbf{Catch} execute $p^*$ during $[\tnb,\tf]$, then \textbf{break}
            \STATE \quad  $(\zpos_0,\zvel_0)\gets\texttt{StatePrediction}(z_0,p^*,\tnb)$
            \STATE \quad $z_0\gets [(\zpos_0)^\top,(\zvel_0)^\top]^\top$
            \STATE \quad \textbf{If} ($\zvel_0\notin\Zvel_0$), execute $p^*$ during $[\tnb,\tf]$ and \textbf{break}
            \STATE \quad  \textbf{Reset} $t$ to 0
        \STATE \textbf{End}
    \end{algorithmic}
\end{algorithm}




\section{Extensions} \label{sec:implementation}

This section describes how to extend various components of \methodname{}.
This section begins by describing how to apply CORA to compute tight, conservative approximations of the FRS.
Next, it illustrates how to verify the satisfaction of Assumption \ref{ass: small slip}.
The section concludes by describing how to apply \methodname{} to AWD and RWD vehicles.

\subsection{Subdivision of Initial Set and Families of Trajectories}
\label{subsec: bin idea}

In practice, CORA may generate overly conservative representations for the FRS if the initial condition set is large. 
To address this challenge, one can instead partition $\Z_0$ and $\P$ and compute a FRS beginning from each element in this partition. 
Note one could then still apply \methodname{} as in Algorithm \ref{alg:methodname}.
However in Line 5 must solve multiple optimizations of the form \opt in parallel.
Each of these optimizations optimizes over a unique partition element that contains initial condition $z_0$, then $p^*$ is set to be the feasible control parameter that achieves the minimum cost function value among these optimizations.
Similarly note if one had multiple classes of  desired trajectories (e.g. lane change, longitudinal speed changes, etc.) that were each parameterized in distinct ways, then one could extend \methodname{} just as in the instance of having a partition of the initial condition set. 
In this way one could apply \methodname{} to optimize over multiple families of desired trajectories to generate not-at-fault behavior.
Note, that the planning horizon $\tf$ is constant within each element of the partition, but can vary between different elements in the partition.

\subsection{Satisfaction of Assumption \ref{ass: small slip}}
\label{sec: satisfaction of linear regime}

Throughout our analysis thus far, we assume that the slip ratios and slip angles stay within the linear regime as described in Assumption \ref{ass: small slip}. 
This subsection describes how to ensure that Assumption \ref{ass: small slip} is satisfied by performing an offline verification on the computed reachable sets. 


Recall that in an FWD vehicle model, $F_\text{xr}(t) = 0$ for all $t$ as in Remark \ref{rem: FWD Fxr}.
By plugging \eqref{eq: linear long tire force} in \eqref{eq: Fxf}, one can derive:
\begin{equation}
\label{eq: lambda_f computation}
    \begin{split}
        \lambda_\text{f}(t) = & \frac{l}{gl_\text{r}\bar\mu}\big(  - K_u u(t) +  K_u \udes(t,p) +  \\
         & + \dudes(t,p) - v(t)r(t) + \tau_u(t,p)\big).
    \end{split}
\end{equation}
Similarly by plugging \eqref{eq: linear lat tire force} in \eqref{eq: Fyf} one can derive:
\begin{equation}
\label{eq: alpha_f computation}
    \begin{split}
       \alpha_\text{f}(t) = &  -\frac{I_\text{zz} K_r}{l_\text{f}\bar c_{\alpha \text{f}}}\left( r(t) - \rdes(t,p) \right) + \\
        &  -\frac{I_\text{zz} K_h}{l_\text{f}\bar c_{\alpha \text{f}}}\left( h(t) - \hdes(t,p)\right)+\\
        &  +\frac{I_\text{zz}}{l_\text{f}\bar c_{\alpha \text{f}}} \drdes(t,p) + \frac{l_\text{r}}{l_\text{f}\bar c_{\alpha \text{f}}} F_\text{yr}(t) + \frac{I_\text{zz}}{l_\text{f}\bar c_{\alpha \text{f}}} \tau_r(t,p).
    \end{split}
\end{equation}
If the slip ratio and slip angle computed in \eqref{eq: lambda_f computation} and \eqref{eq: alpha_f computation} satisfy Assumption \ref{ass: small slip}, they achieve the expected tire forces as introduced in Section \ref{subsec: controller design}.

By Definition \ref{def: zonotope} any $\RR_j = \zonocg{c_{\RR_j}}{G_{\RR_j}}$ that is computed by CORA under the hybrid vehicle dynamics model from a partition element in Section \ref{subsec: bin idea}, can be bounded by a multi-dimensional box $\Int(c_{\RR_j}-|G_{\RR_j}|\cdot\mathbf{1}, c_{\RR_j}+|G_{\RR_j}|\cdot\mathbf{1})$ where $\mathbf{1}$ is a column vector of ones.
This multi-dimensional box gives interval ranges of all elements in $\zaug$ during $T_j$, which allows us to conservatively estimate  $\{|\alpha_\text{r}(t)|\}_{t\in T_j}$, $\{\F_\text{yr}(t)\}_{t\in T_j}$ and $\{|\lambda_\text{f}(t)|\}_{t\in T_j}$ via \eqref{eq: alpha_r def}, \eqref{eq: linear lat tire force} and \eqref{eq: lambda_f computation} respectively using Interval Arithmetic \cite{hickey2001interval}.
The approximation of $\{\F_\text{yr}(t)\}_{t\in T_j}$ makes it possible to over-approximate $\{|\alpha_\text{f}(t)|\}_{t\in T_j}$ via \eqref{eq: alpha_f computation}.

Note in \eqref{eq: lambda_f computation} and \eqref{eq: alpha_f computation} integral terms are embedded in $\tau_u(t,p)$ and $\tau_r(t,p)$ as described in \eqref{eq: tau_u def} and \eqref{eq: tau_r def}.
Because it is nontrivial to perform Interval Arithmetic over integrals, we extend $\zaug$ to $\zaugp$ by appending three more auxiliary states $\varepsilon_u(t) := \int_{t_0}^t \|u(s)-\udes(s,p)\|^2ds$, $\varepsilon_r(t) := \int_{t_0}^t \|r(s)-\rdes(s,p)\|^2ds$ and $\varepsilon_h(t) := \int_{t_0}^t \|h(s)-\hdes(s,p)\|^2ds$. 
Notice
\begin{equation}
    \begin{bmatrix}\dot\varepsilon_u(t) \\ \dot\varepsilon_r(t) \\ \dot\varepsilon_h(t)\end{bmatrix} = \begin{bmatrix}
        \|u(t)-\udes(t,p)\|^2 \\ \|r(t)-\rdes(t,p)\|^2 \\  \|h(t)-\hdes(t,p)\|^2
    \end{bmatrix},
\end{equation}
then we can compute a higher-dimensional FRS of $\zaugp$ during $[0,\tf]$ through the same process as described in Section \ref{sec: rtd}.
This higher-dimensional FRS makes over-approximations of $\{\varepsilon_u(t)\}_{t\in\T_j}$, $\{\varepsilon_r(t)\}_{t\in\T_j}$ and $\{\varepsilon_h(t)\}_{t\in\T_j}$ available for computation in \eqref{eq: lambda_f computation} and \eqref{eq: alpha_f computation}.

If the supremum of $\{|\lambda_\text{f}(t)|\}_{t\in T_j}$ exceeds $\lambdac$ or any supremum of $\{|\alpha_\text{f}(t)|\}_{t\in T_j}$ and $\{|\alpha_\text{r}(t)|\}_{t\in T_j}$ exceeds $\alphac$, then the corresponding partition section of $\Z_0\times\P$ may result in a system trajectory that violates Assumption \ref{ass: small slip}.
Therefore to ensure not-at-fault, we only run optimization over partition elements whose FRS outer-approximations satisfy Assumption \ref{ass: small slip}.
Finally we emphasize that such verification of Assumption \ref{ass: small slip} over each partition element that is described in Section \ref{subsec: bin idea} can be done offline.

\subsection{Generalization to All-Wheel-Drive and Rear-Wheel-Drive}
\label{sec: AWD}

This subsection describes how \methodname{} can be extended to AWD and RWD vehicles. 
AWD vehicles share the same dynamics as \eqref{eq:highspeed_noise} in Section \ref{sec: dynamics} with one exception.
In an AWD vehicle, only the lateral rear tire force is estimated and all the other three tire forces are controlled by using wheel speed and steering angle.
In particular, computations related to the lateral tire forces as \eqref{eq: Fyf} and \eqref{eq: alpha_f computation} are identical to the FWD case .
However, both the front and rear tires contribute nonzero longitudinal forces, and 
they can be specified by solving the following system of linear equations:
\begin{align}
    l_\text{f}F_\text{xf}(t) &= l_\text{r}F_\text{xr}(t) \nonumber \\
    F_\text{xf}(t) + F_\text{xr}(t) &=  -m K_u u(t) + m K_u \udes(t,p) + 
\label{eq: solve Fx in AWD}\\
    &  + m\dudes(t,p) - mv(t)r(t) + m\tau_u(t,p) \nonumber
\end{align}
Longitudinal tire forces $F_\text{xf}(t)$ and $F_\text{xr}(t)$ computed from \eqref{eq: solve Fx in AWD} can then be used to compute wheel speed $\omega_\text{f}(t) = \omega_\text{r}(t)$ as in \eqref{eq: compute w_cmd}.
In this formulation, \eqref{eq: lambda_f computation} also needs to be modified to
\begin{equation}
\begin{split}
    \lambda_\text{f}(t) = \lambda_\text{r}(t) = &  \frac{1}{g\bar\mu}\big(  - K_u u(t) +  K_u \udes(t,p) +  \\
         & + \dudes(t,p) - v(t)r(t) + \tau_u(t,p)\big)
\end{split}
\end{equation}
to verify Assumption \ref{ass: small slip} along the longitudinal direction.
Compared to FWD, in RWD the longitudinal front tire force is $0$ and the longitudinal rear tire force is controlled.
Thus one can generalize to RWD by switching all related computations on $F_\text{xf}(t)$ and $F_\text{xr}(t)$ from the FWD case.



\section{Experiments}
\label{sec:experiment}
This section describes the implementation and evaluation of \methodname{} in simulation using a FWD, full-size vehicle model and on hardware using an AWD, $\frac{1}{10}$th size race car model.
Readers can find a link to the software implementation\footnote{\urlx{https://github.com/roahmlab/REFINE}} and videos\footnote{\urlx{https://drive.google.com/drive/folders/1bXl07gTnaA3rJBl7J05SL0tsfIJEDfKy?usp=sharing}, \urlx{https://drive.google.com/drive/folders/1FvGHuqIRQpDS5xWRgB30h7exmGTjRyel?usp=sharing}} online.

\subsection{Desired Trajectories}
\label{subsec: desired traj}


As detailed in Section \ref{subsec: controller design}, the proposed controller relies on desired trajectories of vehicle longitudinal speed and yaw rate satisfying Definition \ref{def:traj_param}.
To test the performance of the proposed controller and planning framework, we selected $3$ families of desired trajectories that are observed during daily driving.
Each desired trajectory is the concatenation of a driving maneuver and a contingency braking maneuver. 
The  driving maneuver is either a \emph{speed change}, \emph{direction change}, or \emph{lane change} (i.e. each option corresponds to one of the $3$ families of desired trajectories).
Moreover, each desired trajectory is parameterized by $p=[p_u,p_y]^\top\in\P\subset\R^2$ where $p_u$ denotes desired longitudinal speed, and $p_y$ decides desired lateral displacement. 

Assuming that the ego vehicle has initial longitudinal speed $u_0\in\R$ at time $0$, the desired trajectory for longitudinal speed is the same for each of the $3$ families of desired trajectories:
\begin{equation}
\label{eq: u_des}
    \udes(t,p) = \begin{cases}
        u_0+\frac{p_u - u_0}{\tnb}t, \hspace{0.2cm} \text{ if } 0<t<\tnb\\
        \ubrk(t,p), \hspace{0.74cm} \text{ if } t\geq\tnb
    \end{cases}
\end{equation}
where 
\begin{equation}
\label{eq: ubrk}
    \ubrk(t,p) = \begin{cases}
        p_u + (t-\tnb)\amax, \\
        \hspace{0.5cm}\text{if } p_u>\uc \text{ and } \tnb\leq t<\tnb+\frac{\uc-p_u}{\amax}\\
        0,\hspace{0.15cm}\text{if }p_u>\uc \text{ and }t\geq\tnb+\frac{\uc-p_u}{\amax}\\
        0,\hspace{0.15cm}\text{if }p_u\leq\uc \text{ and }t\geq\tnb
    \end{cases}
\end{equation}
with some deceleration $\amax<0$.
Note by Definition \ref{def:traj_param} $\tstop$ can be specified as
\begin{equation}
    \tstop = \begin{cases} \tnb+\frac{\uc-p_u}{\amax}, ~\text{ if }p_u>\uc\\
    \tnb, \hspace{1.55cm}\text{ if }p_u\leq\uc.
    \end{cases}
\end{equation}
The desired longitudinal speed approaches $p_u$ linearly from $u_0$ before braking begins at time $\tnb$, then decreases to $\uc$ with deceleration $\amax$ and immediately drops down to 0 at time $\tstop$.
Moreover, one can verify that the chosen $\udes(t,p)$ in \eqref{eq: u_des} satisfies the assumptions on desired longitudinal speed in Lemma \ref{lem:braking}.



\begin{figure*}[!htb]
    \centering
    \includegraphics[trim={0cm, 6cm, 0cm, 0.4cm},clip,width=\textwidth]{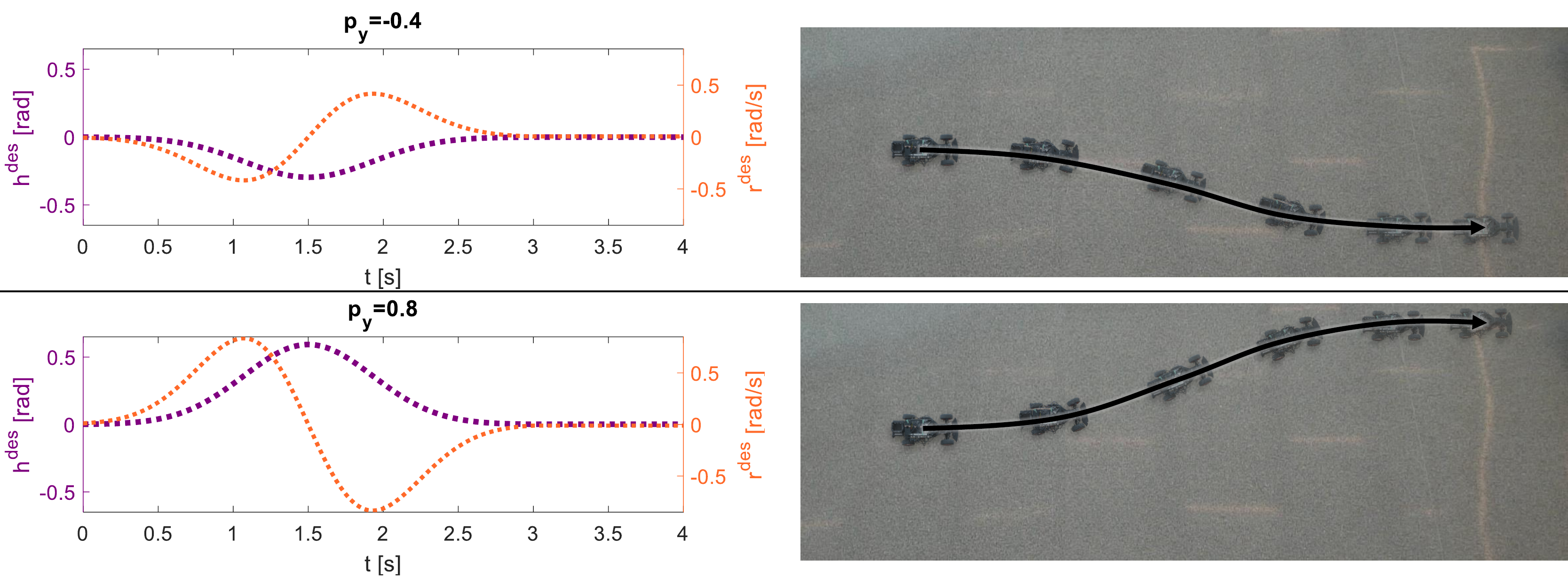}
    \caption{Examples of $\hdes(t,p)$ and $\rdes(t,p)$ to achieve lane changes with $u_0=1.0$ [m/s], $\tnb=3.0$ [s], $\hdes_1=\frac{20}{27}$, $\hdes_2=\frac{27}{10}$, and $p_y$ taking values of -0.4 and 0.8 from top to bottom.
    Note $p_u$ is set as $u_0$ to maintain the vehicle longitudinal speed before $\tnb$ among both examples.}
    \label{fig: des lan change}
    \vspace*{-0.5cm}
\end{figure*}

Assuming the ego vehicle has initial heading $h_0\in[-\pi,\pi]$ at time $0$, the desired heading trajectory varies among the different trajectory families.
Specifically, for the trajectory family associated with speed change:
\begin{equation}
    \hdes(t,p) = h_0, ~ \forall t\geq0.
\end{equation}
Desired heading trajectory for the trajectory family associated with direction change:
\begin{equation}
    \hdes(t,p) = \begin{cases}
        h_0+\frac{p_yt}{2}-\frac{p_y\tnb}{4\pi}\sin\left(\frac{2\pi t}{\tnb} \right), \text{ if } 0\leq t<\tnb\\
        h_0+\frac{p_y\tnb}{2}, \hspace{2.45cm}\text{ if } t\geq\tnb
    \end{cases}
\end{equation}
and for the trajectory family associated with lane change:
\begin{equation}
    \hdes(t,p) = \begin{cases}
        h_0 + \hdes_1 p_y \cdot \mathrm{e}^{-\hdes_2(t-0.5\tnb)^2},\\ \hspace{3.8cm}\text{ if } 0\leq t<\tnb\\
        h_0, \hspace{2.9cm}~~~ \text{ if } t\geq\tnb
    \end{cases}
\end{equation}
where $\mathrm{e}$ is Euler's number, and $\hdes_1$ and $\hdes_2$ are user-specified auxiliary constants that adjust the desired heading amplitude.
Illustrations of speed change and direction change maneuvers can be found in the software repository\footnote{\urlx{https://github.com/roahmlab/REFINE/blob/main/Rover_Robot_Implementation/README.md\#1-desired-trajectories}}.
As shown in Figure \ref{fig: des lan change}, $\hdes(t,p)$ remains constant for all $t\geq\tnb$ among all families of desired trajectories.
By Definition \ref{def:traj_param}, desired trajectory of yaw rate is set as $\rdes(t,p)=\frac{d}{dt}\hdes(t,p)$ among all trajectory families.

In this work, $\tnb$ for the speed change and direction change trajectory families are set equal to one another.
$\tnb$ for the lane change trajectory family is twice what it is for the direction change and speed change trajectory families.
This is because a lane change can be treated as a concatenation of two direction changes. 
Because we do not know which desired trajectory ensures not-at-fault {\textit a priori}, during each planning iteration, to guarantee real-time performance, $\tplan$ should be no greater than the smallest duration of a driving maneuver, i.e. speed change or direction change.


\subsection{Simulation on a FWD Model}
\label{subsec: simulation}

This subsection describes the evaluation of \methodname{} in simulation. 
In particular, this section describes the simulation environment, how we implement \methodname{}, the methods we compare it to, and the results of the evaluation. 

\subsubsection{Simulation Environment}

We evaluate the performance on $1000$ randomly generated $3$-lane highway scenarios in which the same full-size, FWD vehicle as the ego vehicle is expected to autonomously navigate through dynamic traffic for $1$[km] from a fixed initial condition.
All lanes of all highway scenario share the same lane width as $3.7$[m].
Each highway scenario contains up to $24$ moving vehicles and up to $5$ static vehicles that start from random locations and are all treated as obstacles to the ego vehicle.
Moreover, each moving obstacle maintains its randomly generated highway lane and initial speed up to $25$[m/s] for all time.
Because each highway scenario is randomly generated, there is no guarantee that the ego vehicle has a path to navigate itself from the start to the goal.
Such cases allow us to verify if the tested methods can still keep the ego vehicle safe even in infeasible scenarios. 
Parameters of the ego vehicle can be found in the software implementation readme\footnote{\urlx{https://github.com/roahmlab/REFINE/blob/main/Full_Size_Vehicle_Simulation/README.md\#vehicle-and-control-parameters}}.

During each planning iteration, all evaluated methods use the same high level planner.
This high level planner generates waypoints by first choosing the lane on which the nearest obstacle ahead has the largest distance from the ego vehicle.
Subsequently it picks a waypoint that is ahead of the ego vehicle and stays along the center line of the chosen lane.
The cost function in \opt or in any of the evaluated optimization-based motion planning algorithms is set to be the Euclidean distance between the waypoint generated by the high level planner and the predicted vehicle location based on initial state $z_0$ and decision variable $p$.
All simulations are implemented and evaluated in MATLAB R2022a on a laptop with an Intel i7-9750H processor and 16GB of RAM.

\subsubsection{\methodname{} Simulation Implementation}
\methodname{} invokes C++ for the online optimization using IPOPT \cite{wachter2006implementation}.
Parameters of \methodname{}'s controller are chosen to satisfy the conditions in Lemma \ref{lem:braking} and can be found in the software implementation readme\footnote{\urlx{https://github.com/roahmlab/REFINE/blob/main/Full_Size_Vehicle_Simulation/README.md\#vehicle-and-control-parameters}}.
\methodname{} tracks families of desired trajectories as described in Section \ref{subsec: desired traj} with $\P= \{(p_u,p_y)\in[5,30]\times[-0.8,0.8]\mid p_u=u_0 \text{ if } p_y\neq0\}$,  $\amax=-5.0[\text{m}/\text{s}^2]$, $\hdes_1=\frac{6\sqrt{2\mathrm{e}}}{11}$ and $\hdes_2 = \frac{121}{144}$.
The duration $\tnb$ of driving maneuvers for each trajectory family is 3[s] for speed change, 3[s] for direction change and 6[s] for lane change, therefore $\tplan$ is set to be 3[s].
As discussed in Section \ref{subsec: bin idea}, during offline computation, we evenly partition the first and second dimensions of $\P$ into intervals of lengths $0.5$ and $0.4$, respectively.
For each partition element, $\tf$ is assigned to be the maximum possible value of $\tb$ as computed in \eqref{eq: tb computation} in which $\tfstop$ is by observation no greater than 0.1[s].
An outer-approximation of the FRS is computed for every partition element of $\P$ using CORA with $\Delta t$ as $0.015$[s], $0.010$[s], $0.005$[s] and $0.001$[s].
Note, that we choose these different values of $\Delta t$ to highlight how this choice affects the performance of \methodname{}. 

\subsubsection{Other Implemented Methods}
We compare \methodname{} against several state of the art trajectory planning methods: a baseline zonotope reachable set method \cite{manzinger2020using}, a Sum-of-Squares-based RTD (SOS-RTD) method \cite{kousik2020bridging}, and an NMPC method using GPOPS-II \cite{patterson2014gpops}. 



The first trajectory planning method that we implement is a baseline zonotope based reachability method that selects a finite number of possible trajectories rather than a continuum of possible trajectories as \methodname{} does. 
This baseline method is similar to the classic funnel library approach to motion planning \cite{majumdar2017funnel} in that it chooses a finite number of possible trajectories to track. 
The baseline method computes zonotope reachable sets using CORA with $\Delta t=0.010$[s] over a sparse discrete control parameter space $\Psparse:=\{(p_u,p_y)\in\{5,5.1,5.2,\ldots,30\}\times\{0,0.4\} \mid p_u = u_0 \text{ if }p_y\neq 0\}$ and a dense discrete control parameter space $\Pdense:=\{(p_u,p_y)\in\{5,5.1,5.2,\ldots,30\}\times\{0,0.04,0.08,\ldots,0.8\} \mid p_u = u_0 \text{ if }p_y\neq 0\}$.
We use $\Psparse$ and $\Pdense$ to illustrate the challenges associated with applying this baseline method in terms of computation time, memory consumption, and the ability to robustly travel through complex simulation environments. 
During each planning iteration, the baseline method searches through the discrete control parameter space until a feasible solution is found such that the corresponding zonotope reachable sets have no intersection with any obstacles over the planning horizon.
The search procedure over this discrete control space is biased to select the same trajectory parameter that worked in the prior planning iteration or to search first from trajectory parameter that are close to one that worked in the previous planning iteration.

The SOS-RTD plans a controller that also tracks families of trajectories to achieve speed change, direction change and lane change maneuvers with braking maneuvers as described in Section \ref{subsec: desired traj}.
SOS-RTD offline approximates the FRS by solving a series of polynomial optimizations using Sum-of-Squares so that the FRS can be over-approximated as a union of superlevel sets of polynomials over successive time intervals of duration 0.1[s] \cite{kousik2020bridging}.
Computed polynomial FRS are further expanded to account for footprints of other vehicles offline in order to avoid buffering each obstacle with discrete points online \cite{vaskov2019towards}. 
During online optimization, SOS-RTD plans every $3$[s] and uses the same cost function as \methodname{} does, but checks collision against obstacles by enforcing that no obstacle has its center stay inside the FRS approximation during any time interval.

The NMPC method does not perform offline reachability analysis. 
Instead, it directly computes the control inputs that are applicable for $\tnb$ seconds by solving an optimal control problem. 
This optimal control problem is solved using GPOPS-II in a receding horizon fashion.
The NMPC method conservatively ensures collision-free trajectories by covering the footprints of the ego vehicle and all obstacles with two overlapping balls, and requiring that no ball of the ego vehicle intersects with any ball of any obstacle.
Notice during each online planning iteration, the NMPC method does not need pre-defined desired trajectories for solving control inputs.
Moreover, it does not require the planned control inputs to stop the vehicle by the end of planned horizon as the other three methods do.

\begin{table*}[!htb]
    \centering
    \begin{tabular}{|c||c|c|c|c|c|c|}
    \hline 
    \multirow{2}{*}{Method} & \multirow{2}{*}{Safely Stop} & \multirow{2}{*}{Crash} & \multirow{2}{*}{Success} & \multirow{2}{*}{Average Travel Speed} & Solving Time of Online Planning & \multirow{2}{*}{Memory} \\ 
    & & & & & (Average, Maximum) &\\\hline
    Baseline (sparse, $\Delta t=0.010$) & 38\% & {\bf 0\%} & 62\% & 22.3572[m/s] &  (2.03[s], 4.15[s]) &  980 MB \\ \hline
    Baseline (dense, $\Delta t=0.010$) & 30\% & {\bf 0\%} & 70\% & 23.6327[m/s] &  (12.42[s], 27.74[s]) &  9.1 GB\\ \hline
    SOS-RTD & 36\% & {\bf 0\%} & 64\% & 24.8049[m/s] &  ({\bf 0.05[s]}, 1.58[s]) & 2.4 GB\\ \hline
    NMPC & 3\% & 29\% & 68\% & 27.3963[m/s] &  (40.89[s], 534.82[s]) & N/A \\ \hline
    \methodname{} ($\Delta t=0.015$) & 27\% & {\bf 0\%} & 73\% & 23.2452[m/s] &  (0.34[s], {\bf 0.95[s]}) &  {\bf 488 MB}\\ \hline
    \methodname{} ($\Delta t=0.010$) & 17\% & {\bf 0\%}& 83\% & 24.8311[m/s] &  (0.52[s], 1.57[s]) &  703 MB\\ \hline
    \methodname{} ($\Delta t=0.005$) & 16\% & {\bf 0\%} & {\bf 84}\% & 24.8761[m/s] &  (1.28[s], 4.35[s]) &  997 MB\\ \hline
    \methodname{} ($\Delta t=0.001$) & 16\% & {\bf 0\%} & {\bf 84}\% & 24.8953[m/s] &  (6.48[s], 10.78[s]) &  6.4 GB\\ \hline
    \end{tabular}
    \caption{Summary of performance of various tested techniques on the same $1000$ simulation environments.}
    \label{table: simulation result}
    \vspace*{-0.5cm}
\end{table*}





\subsubsection{Evaluation Criteria}
We evaluate each implemented trajectory planning method in several ways as summarized in Table \ref{table: simulation result}
First, we report the percentage of times that each planning method either came safely to a stop (in a not-at-fault manner), crashed, or successfully navigated through the scenario. 
Note a scenario is terminated when one of those three conditions is satisfied.
Second, we report the average travel speed during all scenarios.
Third, we report the average and maximum planning time over all scenarios.
Finally, we report on the size of the pre-computed reachable set.

\subsubsection{Results}
 
\methodname{} achieves the highest success rate among all evaluated methods and has no crashes. 
The success rate of \methodname{} converges to 84\% as the value of $\Delta t$ decreases because the FRS approximation becomes tighter with denser time discretization. 
However as the time discretization becomes finer, memory consumption grows larger because more zonotopes are used to over-approximate FRS.
Furthermore, due to the increasing number of zonotope reachable sets, the solving time also increases and begins to exceed the allotted planning time. 
According to our simulation, we see that $\Delta t=0.010$[s] results in high enough successful rate while maintaining a planning time no greater than $3$[s].

The baseline method with $\Psparse$ shares almost the same memory consumption as \methodname{} with $\Delta t=0.005$[s], but results in a much lower successful rate and smaller average travel speed.
When the baseline method runs over $\Pdense$, its success rate is increased, but still smaller than that of \methodname{}.
More troublingly, its memory consumption increases to $9.1$ GB.
Neither evaluated baseline is able to finish online planning within $3$[s].
Compared to \methodname, SOS-RTD completes online planning faster and can also guarantee vehicle safety with a similar average travel speed.
However SOS-RTD needs a memory of $2.4$ GB to store its polynomial reachable sets, and its success rate is only $64$\% because the polynomial reachable sets are more conservative than zonotope reachable sets.

When the NMPC method is utilized for motion planning, the ego vehicle achieves a similar success rate as SOS-RTD, but crashes occur $29$\% of the time.
Note the NMPC method achieves a higher average travel speed of the ego vehicle when compared to the other three methods.
More aggressive operation can allow the ego vehicle drive closer to obstacles, but can make subsequent obstacle avoidance difficult.
The NMPC method uses $40.8906$[s] on average to compute a solution, which makes real-time path planning untenable.

\begin{figure*}[!htb]
    \centering
    \begin{subfigure}[b]{0.95\textwidth}
         \centering
         \includegraphics[trim={0cm, 0cm, 0cm, 0cm},clip,width=\textwidth]{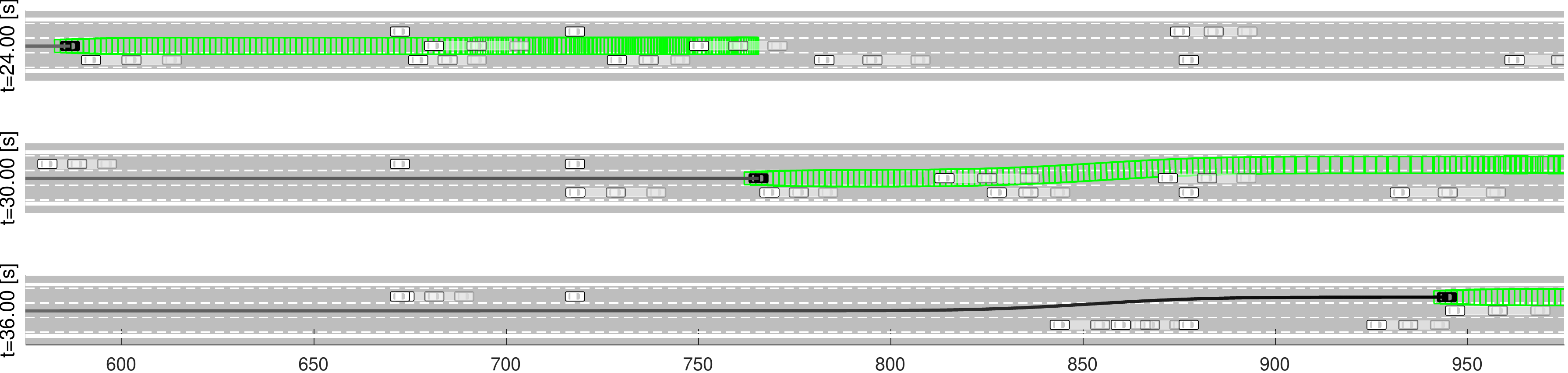}
         \caption{\methodname{} utilized.}
         \label{fig: simulation result 2-flzono}
     \end{subfigure}
    \begin{subfigure}[b]{0.95\textwidth}
         \centering
         \includegraphics[trim={0cm, 0cm, 0cm, -0.5cm},clip,width=\textwidth]{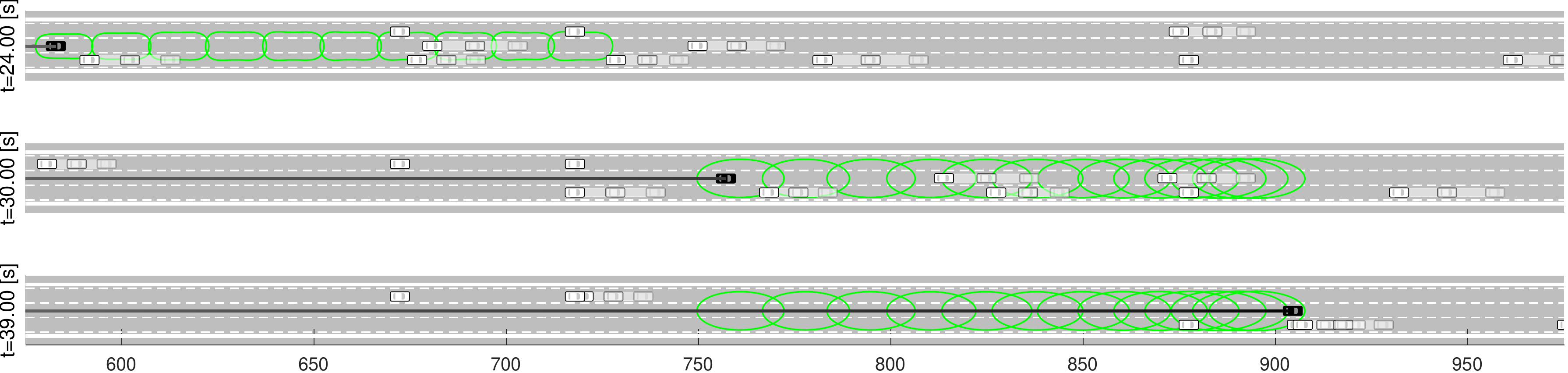}
         \caption{SOS-RTD utilized.}
         \label{fig: simulation result 2-sos}
     \end{subfigure}
    \begin{subfigure}[b]{0.95\textwidth}
         \centering
         \includegraphics[trim={0cm, 0cm, 0cm, -0.5cm},clip,width=\textwidth]{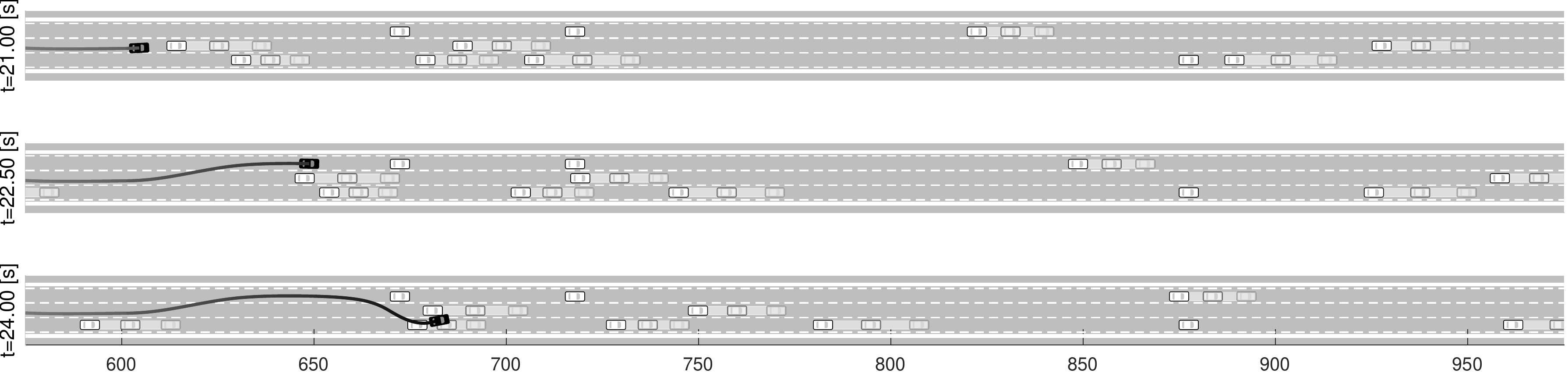}
         \caption{NMPC utilized.}
         \label{fig: simulation result 2-gpops}
     \end{subfigure}
    \caption{An illustration of the performance of \methodname{}, SOS-RTD, and NMPC on the same simulated scenario. 
    In this instance \methodname{} successfully navigates the ego vehicle through traffic (top three images), SOS-RTD stops the ego vehicle to avoid collision due to the conservatism of polynomial reachable sets (middle three images), and NMPC crashes the ego vehicle even though its online optimization claims that it has found a feasible solution (bottom three images).
    In each set of images, the ego vehicle and its trajectory are colored in black.
    Zonotope reachable sets for \methodname{} and polynomial reachable sets for SOS-RTD are colored in green. 
    Other vehicles are obstacles and are depicted in white.
    If an obstacle is moving, then it is plotted at 3 time instances $t$, $t+0.5$ and $t+1$ with increasing transparency.
    Static vehicles are only plotted at time $t$.}
    \label{fig: simulation result 2}
    \vspace*{-0.5cm}
\end{figure*}


Figure \ref{fig: simulation result 2} illustrates the performance of the three methods in the same scene at three different time instances. 
In Figure \ref{fig: simulation result 2-flzono}, because \methodname{} gives a tight approximation of the ego vehicle's FRS using zonotopes, the ego vehicle is able to first bypass static vehicles in the top lane from $t=24$[s] to $t=30$[s], then switch to the top lane and bypass vehicles in the middle lane from $t=30$[s] to $t=36$[s].
In Figure \ref{fig: simulation result 2-sos} SOS-RTD is used for planning.
In this case the ego vehicle bypasses the static vehicles in the top lane from $t=24$[s] to $t=30$[s]. 
However because online planing becomes infeasible due to the conservatism of polynomial reachable sets, the ego vehicle executes the braking maneuver to stop itself $t=30$[s] to $t=36$[s].
In Figure \ref{fig: simulation result 2-gpops} because NMPC is used for planning, the ego vehicle drives at a faster speed and arrives at $600$[m] before the other evaluated methods.
Because the NMPC method only enforces collision avoidance constraints at discrete time instances, the ego vehicle ends up with a crash at $t=24$[s] though NMPC claims to find a feasible solution for the planning iteration at $t=21$[s].


\subsection{Real World Experiments}
\methodname{} was also implemented in C++17 and tested in the real world using a $\frac{1}{10}$th All-Wheel-Drive car-like robot, Rover, based on a Traxxax RC platform. 
The Rover is equipped with a front-mounted Hokuyo UST-10LX 2D lidar that has a sensing range of 10[m] and a field of view of 270\degree.
The Rover is equipped with a VectorNav VN-100 IMU unit which publishes data at 800Hz.
Sensor drivers, state estimator, obstacle detection, and the proposed controller are run on an NVIDIA TX-1 on-board computer.
A standby laptop with Intel i7-9750H processor and 32GB of RAM is used for localization, mapping, and solving \opt in over multiple partitions of $\P$.
The rover and the standby laptop communicate over wifi using ROS \cite{ros}. 

The desired trajectories on the Rover are parameterized with $\P= \{(p_u,p_y)\in[0.05,2.05]\times[-1.396,1.396]\mid p_u=u_0 \text{ if } p_y\neq0\}$, $\amax=-1.5$[m/sec${}^2$], $\hdes_1=\frac{20}{27}$ and $\hdes_2=\frac{27}{10}$ as described in Section \ref{subsec: desired traj}.
The duration $\tnb$ of driving maneuvers for each trajectory family is set to $1.5$[s] for speed change, $1.5$[s] for direction change, and $3$[s] for lane change, thus planning time for real world experiments is set as $\tplan=1.5$[s].
The parameter space $\P$ is evenly partitioned along its first and second dimensions into small intervals of lengths 0.25 and 0.349, respectively.
For each partition element, $\tf$ is set equal to the maximum possible value of $\tb$ as computed in \eqref{eq: tb computation} in which $\tfstop$ is by observation no greater than 0.1[s].
The FRS of the Rover for every partition element of $\P$ is overapproximated using CORA with $\Delta t = 0.01$[s].
During online planning, a waypoint is selected in real time using Dijkstra's algorithm \cite{dijkstra1959note}, and the cost function of \opt is set in the same way as we do in simulation as described in Section \ref{subsec: simulation}.
The robot model, environment sensing, and state estimation play key roles in real world experiments.
In the rest of this subsection, we describe how to bound the modeling error in \eqref{eq:highspeed_noise} and summarize the real world experiments.
Details regarding Rover model parameters, the controller parameters, how the Rover performs localization, mapping, and obstacle detection and how we perform system identification of the tire models can be found in our software implementation readme\footnote{\urlx{https://github.com/roahmlab/REFINE/blob/main/Rover_Robot_Implementation/README.md}}.

\subsubsection{State Estimation and System Identification on Model Error in Vehicle Dynamics}
\label{subsubsec: sysid}

\begin{figure}[t]
    \centering
    \includegraphics[trim={1cm, 9.4cm, 2cm, 10.2cm},clip,width=\columnwidth]{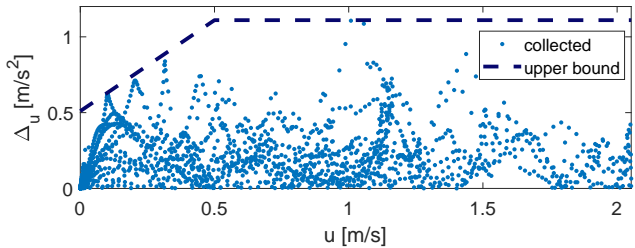}
    \caption{An illustration of the modeling error along the dynamics of $u$. 
    Collected $\Delta_u(t)$ is bounded by $M_u = 1.11$ for all time.
    Whenever $u(t)\leq\uc=0.5$, $\Delta_u(t)$ is bounded by $\bpro u(t) + \boff$ with  $\bpro = 1.2$ and $\boff = 0.51$.}
    \label{fig: Delta u}
    \vspace*{-0.5cm}
\end{figure}    

The modeling error in the dynamics \eqref{eq:highspeed_noise} arise from ignoring aerodynamic drag force and the inaccuracies of state estimation and the tire models.
We use the data collected to fit the tire models to identify the modeling errors $\Delta_u$, $\Delta_v$, and $\Delta_r$. 

  
We compute the model errors as the difference between the \emph{actual accelerations} collected by the IMU and the \emph{estimation of applied accelerations} computed via \eqref{eq:highspeed_perfect} in which tire forces are calculated via \eqref{eq: linear long tire force} and \eqref{eq: linear lat tire force}.
The estimation of applied accelerations is computed using the estimated system states via an Unscented Kalman Filter (UKF) \cite{unscented-kalman-filter}, which treats SLAM results, IMU readings, and encoding information of wheel and steering motors as observed outputs of the Rover model.
The robot dynamics that UKF uses to estimate the states is the error-free, high-speed dynamics \eqref{eq:highspeed_perfect} with linear tire models. 
Note the UKF state estimator is still applicable in the low-speed case except the estimation of $v$ and $r$ are ignored. 
To ensure $\Delta_u$, $\Delta_v$ and $\Delta_r$ are square integrable, we set $\Delta_u(t) = \Delta_v(t) = \Delta_r(t) = 0$ for all $t\geq\tb$ where $\tb$ is computed in Lemma \ref{lem:braking}.
As shown in Figure \ref{fig: Delta u} bounding parameters $M_u$, $M_v$, and $M_r$ are selected to be the maximum value of $|\Delta_u(t)|$, $|\Delta_v(t)|$, and $|\Delta_r(t)|$ respectively over all time, and $\bpro$ and $\boff$ are generated by bounding $|\Delta_u(t)|$ from above when $u(t)\leq\uc$. 

\subsubsection{Demonstration}

The Rover was tested indoors under the proposed controller and planning framework in 6 small trials and 1 loop trial\footnote{\urlx{https://drive.google.com/drive/folders/1FvGHuqIRQpDS5xWRgB30h7exmGTjRyel?usp=sharing}}. 
In every small trail, up to $11$ identical $0.3\times0.3\times0.3$[m]${}^3$ cardboard cubes were placed in the scene before the Rover began to navigate itself.
The Rover was not given prior knowledge of the obstacles for each trial.
Figure \ref{fig: hardware demo} illustrates the scene in the 6th small trial and illustrates \methodname{}'s performance.
The zonotope reachable sets over-approximate the trajectory of the Rover and never intersect with obstacles. 

In the loop trial, the Rover was required to perform 3 loops, and each loop is about 100[m] in length.
In the first loop of the loop trial, no cardboard cube was placed in the loop, while in the last two loops the cardboard cubes were randomly thrown at least 5[m] ahead of the running Rover to test its maneuverability and safety.
During the loop trial, the Rover occasionally stoped because a randomly thrown cardboard cube might be close to a waypoint or the end of an executing maneuver.
In such cases, because the Rover was able to eventually locate obstacles more accurately when it was stopped, the Rover began a new planning iteration immediately after stopping and passed the cube when a feasible plan with safety guaranteed was found.

For all 7 real-world testing trials, the Rover either safely finishes the given task, or it stops itself before running into an obstacle if no clear path is found.
The Rover is able to finish all computation of a planning iteration within 0.4021[s] on average and 0.6545[s] in maximum, which are both smaller than $\tplan=1.5$[s], thus real-time performance is achieved.

\begin{figure*}[!htb]
    \centering
    \begin{subfigure}[b]{0.49\textwidth}
         \centering
         \includegraphics[trim={0cm, 0cm, 0cm, 0cm},clip,width=\textwidth]{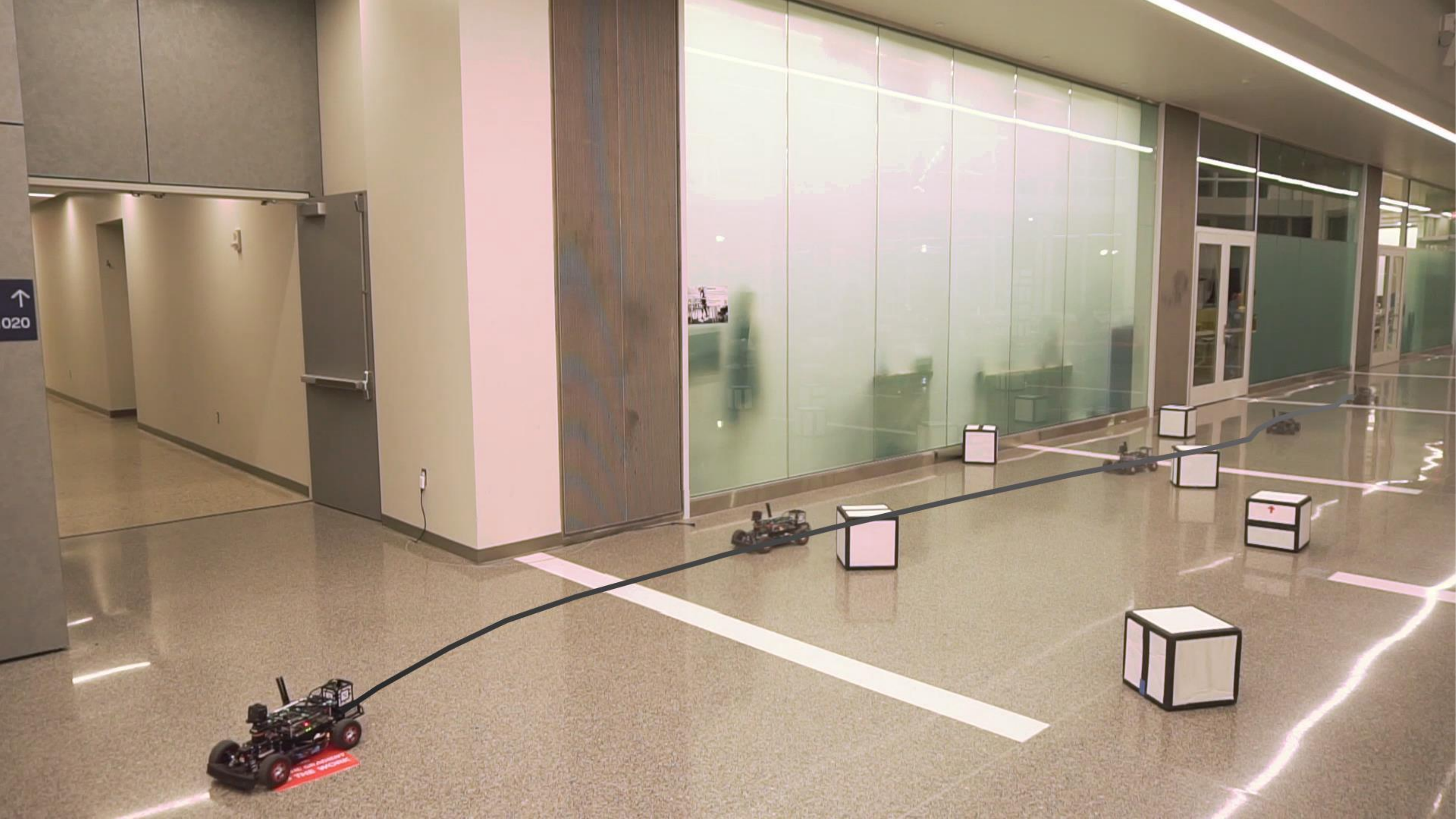}
         \caption{}
     \end{subfigure}
     \hfill
     \begin{subfigure}[b]{0.49\textwidth}
         \centering
         \includegraphics[trim={0.cm, 0cm, 0cm, 0cm},clip,width=\textwidth]{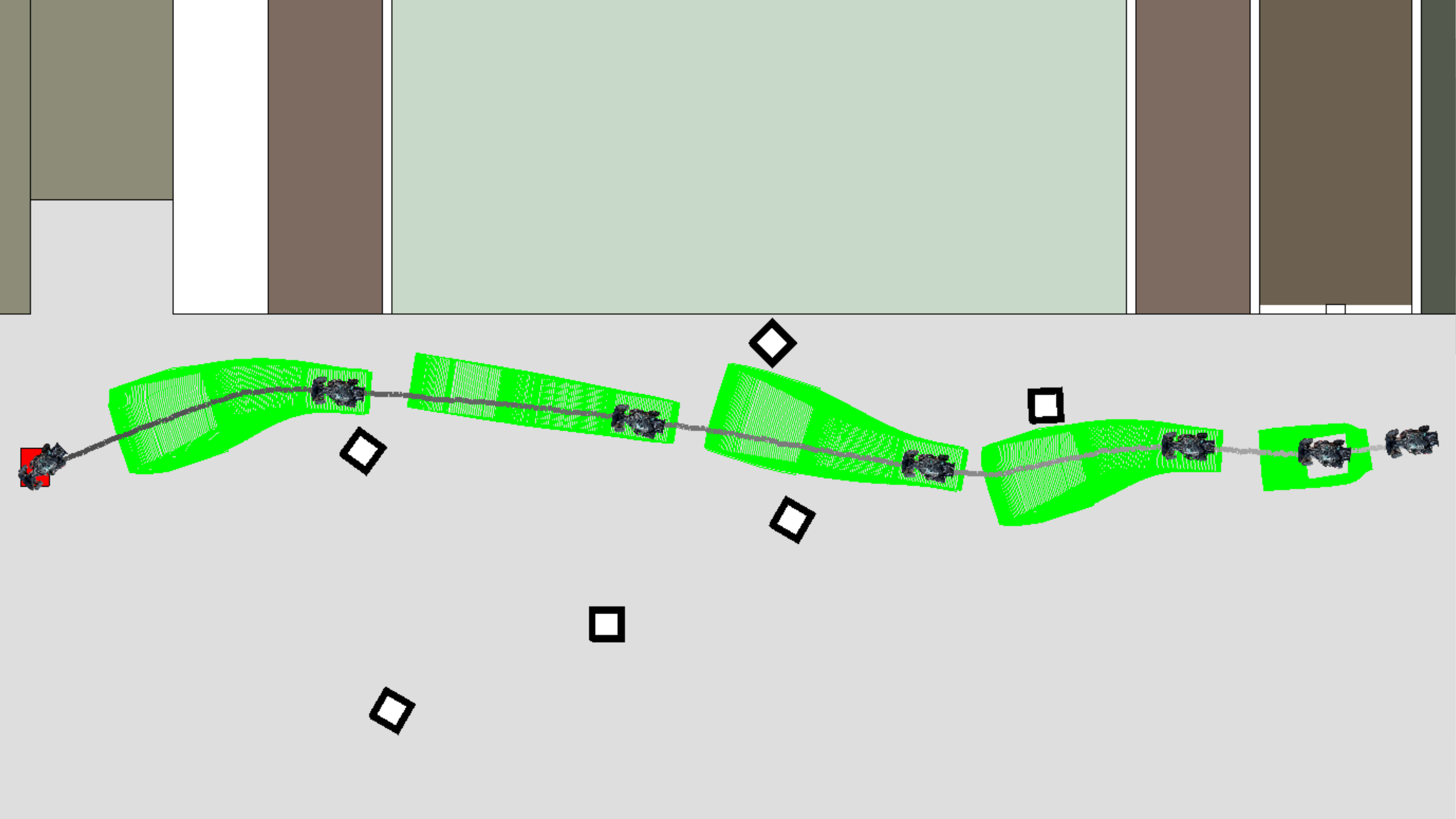}
         \caption{}
     \end{subfigure}
    \caption{An illustration of the performance of \methodname{} during the 6th real-world trial. 
    The rover was able to navigate itself to the goal in red through randomly thrown white cardboard cubes as shown in (a). Online planning using zonotope reachable sets is illustrated in (b) in which trajectory of the Rover is shown from gray to black along time, goal is shown in red, and the zonotope reachable sets at different planning iterations are colored in green.}
    \label{fig: hardware demo}
    \vspace*{-0.5cm}
\end{figure*}

\section{Conclusion}
\label{sec: conclusion}

This work presents a controller-oriented trajectory design framework using zonotope reachable sets.
A robust controller is designed to partially linearize the full-order vehicle dynamics with modeling error by performing feedback linearization on a subset of vehicle states.
Zonotope-based reachability analysis is performed on the closed-loop vehicle dynamics for FRS computation, and achieves less conservative FRS approximation than that of the traditional reachability-based approaches.
Tests on a full-size vehicle model in simulation and a 1/10th race car robot in real hardware experiments show that the proposed method is able to safely navigate the vehicle through random environments in real time and outperforms all evaluated state of the art methods.




 \printbibliography
 
 \appendices
\section{Proof of Lemma \ref{lem:braking}}
\label{app:braking}

\begin{proof}
This proof defines a Lyapunov function candidate and uses it to analyze the tracking error of the ego vehicle's longitudinal speed before time $\tstop$.
Then it describes how $u$ evolves after time $\tstop$ in different scenarios depending on the value of $u(\tstop)$.
Finally it describes how to set the time $\tb$ to guarantee $u(t)=0$ for all $t\geq\tb$.
For convenience, let $\usmall := \frac{M_u}{\kappa_{1,u}M_u+\phi_{1,u}}$,
then by assumption of the theorem $\usmall\in(0.15,\uc]$.
This proof suppresses the dependence on $p$ in $\udes(t,p)$, $\tau_u(t,p)$, $\kappa_u(t,p)$, $\phi_u(t,p)$ and $\eu(t,p)$.

Note by \eqref{eq: e_u def} and rearranging \eqref{eq: u_dot_closed_loop_eq}, 
\begin{equation}
\label{eq: eu dyn}
    \deu(t) = -K_u\eu(t)+\tau_u(t)+\Delta_u(t).
\end{equation}
Recall $\udes$ is piecewise continuously differentiable by Definition \ref{def:traj_param}, so are $\eu$ and $\tau_u$.
Without loss of generality we denote $\{t_1,t_2,\ldots,t_{k_\text{max}}\}$ a finite subdivision of $[0,\tstop)$ with $t_1=0$ and $t_{k_\text{max}}=\tstop$ such that $\udes$ is continuously differentiable over time interval $[t_k,t_{k+1})$ for all $k\in\{1,2,\ldots,k_\text{max}-1\}$.
Define $V(t):= \frac{1}{2}\eu^2(t)$ as a Lyapunov function candidate for $\eu(t)$, then for arbitrary $k\in\{1,2,\ldots,k_\text{max}-1\}$ and $t\in[t_k,t_{k+1})$, one can check that $V(t)$ is always non-negative and $V(t)=0$ only if $\eu(t)=0$.
Then
\begin{align}
    \dot V(t) &= \eu(t)\deu(t)\\
    &= -K_u\eu^2(t)+\eu(t)\tau_u(t)+\eu(t)\Delta_u(t)\\
    &= -K_u\eu^2(t)-(\kappa_u(t)M_u+\phi_u(t))\eu^2(t) + \\ & \hspace*{1cm}+\eu(t)\Delta_u(t)
\end{align}
in which the second equality comes from \eqref{eq: eu dyn} and the third equality comes from \eqref{eq: tau_u def}.
Because the integral terms in \eqref{eq: kappa_u def} and \eqref{eq: phi_u def} are both non-negative, $\kappa_u(t)\geq\kappa_{1,u}$ and $\phi_u(t)\geq\phi_{1,u}$ hold. 
Then
\begin{equation}
\label{ineq: vdot}
    \begin{split}
        \dot V(t) \leq -K_u\eu^2(t)-(\kappa_{1,u}M_u+\phi_{1,u})|\eu(t)|^2+\\
        +|\eu(t)||\Delta_u(t)|.
    \end{split}
\end{equation}
By factoring out $|\eu(t)|$ in the last two terms in \eqref{ineq: vdot}:
\begin{equation}
\label{ineq: dV <= -Ku*eu^2}
    \dot V(t) \leq -K_u\eu^2(t)<0
\end{equation}
holds when $|\eu(t)|>0$ and $|\eu(t)|\geq\frac{|\Delta_u(t)|}{\kappa_{1,u}M_u+\phi_{1,u}}$.
Note $|\eu(t)|\geq\usmall$ conservatively implies $|\eu(t)|\geq\frac{|\Delta_u(t)|}{\kappa_{1,u}M_u+\phi_{1,u}}$ given $|\Delta_u(t)|\leq M_u$ for all time by Assumption \ref{ass: dyn error bnd}.
Then when $|\eu(t)|\geq\usmall>0$ we have \eqref{ineq: dV <= -Ku*eu^2} hold, or equivalently $V(t)$ decreases.  
Therefore if $|\eu(t_k)|\geq\usmall$, $|\eu(t)|$ monotonically decreases during time interval $[t_k,t_{k+1})$ as long as $|\eu(t)|$ does not reach at the boundary of closed ball $\B(0,\usmall)$.
Moreover, if $|\eu(t')|$ hits the boundary of $\B(0,\usmall)$ at some time $t'\in[t_k,t_{k+1})$, $\eu(t)$ is prohibited from leaving the ball for all $t\in[t',t_{k+1})$ because $\dot V(t)$ is strictly negative when $|\eu(t)|=\usmall$.
Similarly $|\eu(t_k)|\leq\usmall$ implies $|\eu(t)|\leq\usmall$ for all $t\in[t_k,t_{k+1})$.

We now analyze the behavior of $\eu(t)$ for all $t\in[0,\tstop)$.
By assumption $\udes(0)=u(0)$, then $|\eu(0)|=0<\usmall$ and thus $|\eu(t)|\leq\usmall$ for all $t\in[t_1,t_2)$.
Because both $u(t)$ and $\udes(t)$ are continuous during $[0,\tstop)$, so is $\eu(t_2)$.
Thus $|\eu(t_2)|\leq\usmall$.
By iteratively applying the same reasoning, one can show that $|\eu(t)|\leq\usmall$ for all $t\in[t_k,t_{k+1})$ and for all $k\in\{1,2,\ldots,k_\text{max}-1\}$, therefore $|\eu(t)|\leq\usmall$ for all $t\in[0,\tstop)$.
Furthermore, because $\udes(t)$ converges to $\uc$ as $t$ converges to $\tstop$ from below, $u(\tstop)\in[\uc-\usmall,\uc+\usmall]$.
Note $u(\tstop)\geq0$ because $\usmall\leq\uc$.

Next we analyze how longitudinal speed of the ego vehicle evolves after time $\tstop$.
Using $V(t)=\frac{1}{2}\eu^2(t)$, we point out that \eqref{ineq: vdot} remains valid for all $t\geq\tstop$, and \eqref{ineq: dV <= -Ku*eu^2} also holds when $|\eu(t)|\geq\usmall$ with $t\geq\tstop$.
Recall $u(t) = \eu(t)$ for all $t\geq\tstop$ given $\udes(t)=0$ for all $t\geq\tstop$, then for simplicity, the remainder of this proof replaces every $\eu(t)$ by $u(t)$ in \eqref{ineq: vdot}, \eqref{ineq: dV <= -Ku*eu^2} and $V(t)$.
Because $u(0)>0$ and $u$ is continuous with respect to time, the longitudinal speed of the ego vehicle cannot decrease from a positive value to a negative value without passing 0.
However when $u(t)=0$, by Assumption \ref{ass: dyn error bnd - low speed} $\Delta_u(t)=0$, thus $\dot u(t)=0$ by \eqref{eq: u_dot_closed_loop_eq} given $\udes(t)=0$ for all $t\geq\tstop$.
In other words, once $u$ arrives at 0, it remains at 0 forever.
For the ease of expression, from now on we assume $t\geq\tstop$ and $u(t)\geq0$ for all $t\geq\tstop$.
Recall $u(\tstop)\in[\uc-\usmall,\uc+\usmall]$ and $\uc-\usmall\in[0,\uc-0.15)$.
We now discuss how $u$ evolves after time $\tstop$ by considering three scenarios, and giving an upper bound of the time at when $u$ reaches 0 for each scenario.

     \noindent \textbf{Case 1 - When $u(\tstop)\leq0.15$:} Because the longitudinal speed stays at 0 once it becomes 0, by Assumption \ref{ass: small speed can stop} the ego vehicle reaches to a full stop no later than $\tfstop + \tstop$.
    
    \noindent \textbf{Case 2 - When $0.15<u(\tstop)\leq\usmall$:} By Assumption \ref{ass: dyn error bnd - low speed}, upper bound of $\dot V(t)$ can be further relaxed from \eqref{ineq: vdot} to
        \begin{equation}
        \label{ineq: dV with qu&bu}
        \begin{split}
            \dot V(t)\leq-K_u u^2(t) - (\kappa_{1,u}M_u+\phi_{1,u}+\\-\bpro)u^2(t)+\boff u(t).
        \end{split}
        \end{equation}
        Moreover, by completing the square among the last two terms in \eqref{ineq: dV with qu&bu}, one can derive
        \begin{equation}
        \label{ineq: dV with u in [0.15,usmall]}
            \dot V(t)\leq -K_u u^2(t) + \frac{(\boff)^2}{4(\kappa_{1,u}M_u+\phi_{1,u}-\bpro)}.
        \end{equation}
        Notice $\frac{(\boff)^2}{4(\kappa_{1,u}M_u+\phi_{1,u}-\bpro)} < 0.15^2 K_u$ by assumption, thus
        \begin{equation}
            \dot V(t) < -K_u(u^2(t)-0.15^2).
        \end{equation}
        This means as long as $u(t)\in[0.15, \uc]$ with $t\geq\tstop$, we obtain $\dot V(t)<0$, or equivalently $V(t) = \frac{1}{2}u^2(t)$ decreases monotonically.
        Recall $u(\tstop)\leq\usmall\leq\uc$, then the longitudinal speed decreases monotonically from $u(\tstop)$ to 0.15 as time increases from $\tstop$.
        Suppose $u$ becomes 0.15 at time $\tb'\geq\tstop$, then $u(t)\leq 0.15$ for all $t\geq\tb'$ because of the fact that $\dot V(t)$ is strictly negative when $u(t)=0.15$.
        
        Define $q_u:=\frac{(\boff)^2}{4(\kappa_{1,u}M_u+\phi_{1,u}-\bpro)}$, then when $u(t)\in[0.15,u(\tstop)]$, \eqref{ineq: dV with u in [0.15,usmall]} can be relaxed to
        \begin{equation}
        \label{ineq: dV with u in [0.15,usmall] +}
            \dot V(t)\leq -K_u\cdot 0.15^2 + q_u.
        \end{equation}
        Integrate both sides of \eqref{ineq: dV with u in [0.15,usmall] +} from time $\tstop$ to $\tb'$ results in
        \begin{equation}
            \tb'\leq\frac{u(\tstop)^2-0.15^2}{2\cdot0.15^2K_u - 2q_u} + \tstop.
        \end{equation}
        Because $u(\tstop)\leq\usmall$, 
        \begin{equation}
        \label{ineq: tb' upper bound}
            \tb'\leq\frac{(\usmall)^2-0.15^2}{2\cdot0.15^2K_u - 2q_u} + \tstop.
        \end{equation}
        Then $u$ becomes 0 no later than time $\tfstop+\sup(\tb')$ based on Assumption \ref{ass: small speed can stop}, where $\sup(\tb')$ as the upper bound of $\tb'$ reads
        \begin{equation}
            \sup(\tb') = \frac{(\usmall)^2-0.15^2}{2\cdot0.15^2K_u - 2q_u} + \tstop.
        \end{equation}
        
     \noindent \textbf{Case 3 - When $\usmall <u(\tstop)\leq \uc+\usmall$:} Recall \eqref{ineq: dV <= -Ku*eu^2} holds given $|\eu(t)|=u(t)\geq\usmall$, then 
        \begin{equation}
        \label{ineq: dV <= -Ku * usmall^2}
            \dot V(t)\leq-K_u\eu^2(t)\leq-K_u(\usmall)^2,
        \end{equation}
        and we have the longitudinal speed monotonically decreasing from $u(\tstop)$ at time $\tstop$ until it reaches at $\usmall$ at some time $\tsmall\geq\tstop$.
        Integrating the left hand side and right hand side of \eqref{ineq: dV <= -Ku * usmall^2} from $\tstop$ to $\tsmall$ gives
        \begin{equation}
        \label{ineq: vdot integration}
            \frac{1}{2}(\usmall)^2 - \frac{1}{2}u(\tstop)^2 \leq -K_u(\usmall)^2(\tsmall - \tstop).
        \end{equation}
        Because $u(\tstop)\leq\uc+\usmall$, \eqref{ineq: vdot integration} results in 
        \begin{equation}
            \tsmall\leq\frac{(\uc+\usmall)^2 - (\usmall)^2}{2K_u(\usmall)^2}+\tstop.
        \end{equation}
        
        Once the longitudinal speed decreases to $\usmall$, we can then follow the same reasoning as in the second scenario for seeking an upper bound of some time $\tb''$ that is no smaller than $\tsmall$ and gives $u(\tb'')=0.15$.
        However, this time we need to integrate both sides of \eqref{ineq: dV with u in [0.15,usmall] +} from time $\tsmall$ to $\tb''$.
        As a result, 
        \begin{equation}
        \label{ineq: tb'' upper bound}
            \tb''\leq\frac{(\usmall)^2-0.15^2}{2\cdot0.15^2K_u - 2q_u} + \tsmall.
        \end{equation}
        Then $u$ becomes 0 no later than time $\tfstop+\sup(\tb'')$ based on Assumption \ref{ass: small speed can stop}, where $\sup(\tb'')$ as the upper bound of $\tb''$ reads
        \begin{equation}
        \begin{split}
            \sup(\tb'') = &\frac{(\usmall)^2-0.15^2}{2\cdot0.15^2K_u - 2q_u} + \\
            &+\frac{(\uc+\usmall)^2 - (\usmall)^2}{2K_u(\usmall)^2}+\tstop.
        \end{split}
        \end{equation}

Now that we have the upper bound for $u$ across these three scenarios, recall that once $u$ arrives at 0, it remains at 0 afterwards, and notice $\sup(\tb'')>\sup(\tb')>\tstop$.
Considering all three scenarios discussed above, setting $\tb$ as the maximum value among $\tfstop+\tstop$, $\tfstop+\sup(\tb')$ and $\tfstop+\sup(\tb'')$, i.e.,
\begin{equation}
\label{eq: tb computation}
    \begin{split}
        \tb = &\tfstop+\frac{(\usmall)^2-0.15^2}{2\cdot0.15^2K_u - 2q_u} + \\
        &+\frac{(\uc+\usmall)^2 - (\usmall)^2}{2K_u(\usmall)^2}+\tstop
    \end{split}
\end{equation}
guarantees that $u(t)=0$ for all $t\geq\tb$. 
\end{proof}
\section{Proof of Theorem \ref{thm: slice_is_good}}
\label{app:slicing}

\begin{proof}
    Because $\zvel_0$ and $p$ have zero dynamics in $HS$, the last $3+n_p$ dimensions in $\RR_j$ are identical to $\Zvel_0\times\P$ for all $j\in \J$. 
    A direct result of Proposition \ref{prop: sliceable} and Definition \ref{defn:slice} is $\Zvel_0\times\P = \zonocg{c_j'}{G_j'}$ where $c_j' = \big[[c_{\RR_j}]_7, [c_{\RR_j}]_8, \ldots, [c_{\RR_j}]_{(9+n_p)}\big]^\top$ and $G_j' = \diag\left(\big[[g_{\RR_{j,1}}]_7, [g_{\RR_{j,2}}]_8, \ldots, [g_{\RR_{j,(3+n_p)}}]_{(9+n_p)}   \big]\right)$ for all $j\in\J$.
    Because $\zvel_0\in \Zvel_0$ and $p\in\P$, then $\frac{[\zvel_0]_{(k-6)}-[c_{\RR_j}]_k}{[g_{\RR_j,(k-6)}]_k} \in [-1,1]$ for all $k\in\{7,8,9\}$, and $\frac{[p]_{(k-9)}-[c_{\RR_j}]_k}{[g_{\RR_j,(k-6)}]_k}\in[-1,1]$ for all $k\in\{10,11,\ldots,(9+n_p)\}$ by Definition \ref{def: zonotope}.
    $\slice( \RR_,\zvel_0,p)$ is generated by specifying the coefficients of the first $3 + n_p$ generators in $\RR_j$ via \eqref{eq: slice center}, thus $\slice( \RR_j,\zvel_0,p)\subset \RR_j$. 
    
    If a solution of $HS$ has initial velocity $\zvel_0$ and control parameter $p$, then the last $3+n_p$ dimensions in $\zaug$ are fixed at $[(\zvel_0)^\top, p^\top]^\top$ for all $t\in T_j$ because of \eqref{eq: dyn tilde_z}.
    $\RR_j$ is generated from CORA, so $\zaug(t)\in \RR_j$ for all $t \in T_j$ by Theorem \ref{thm: FRS over-approximation}, which proves the result.
\end{proof}
\section{Proof of Lemma \ref{lem:footprint}}
\label{app:footprint}

Before proving Lemma \ref{lem:footprint}, we prove the following lemma:
\begin{lem}
\label{lem: slice and rot}
    Let $\RR_j$ be the zonotope computed by CORA under the hybrid vehicle dynamics model beginning from $\Zaug_0$ for arbitrary $j\in\J$.
    Then for any $\zvel_0\in\Zvel_0$ and $p\in\P$
    \begin{equation}
    \label{eq: slice and rot}
        \begin{split}
            \slice\big(\RR_j\oplus & \ROT(\pi_h(\RR_j)), \zvel_0, p\big) = \ROT(\pi_h(\RR_j)) \oplus \\
            & \oplus \slice(\RR_j, \zvel_0, p).
        \end{split}    
    \end{equation}
\end{lem}
\begin{proof}
    Because $\ROT(\pi_h(\RR_j))$ is independent of $\zvel_0$ and $p$ by definition, $\RR_j$ shares the same sliceable generators as $\RR_j\oplus\ROT(\pi_h(\RR_j))$.
    The slice operator only affects sliceable generators, thus \eqref{eq: slice and rot} holds.
\end{proof}

Now we prove Lemma \ref{lem:footprint}:

\begin{proof}
By definition $\slice(\RR_j, \zvel_0, p)$ and $\ROT(\pi_h(\RR_j))$ are both zonotopes, thus $\slice\big(\RR_j\oplus\ROT(\pi_h(\RR_j)), \zvel_0, p\big)$ is a zonotope per \eqref{eq: slice and rot}.
For simplicity denote $\slice\big(\RR_j\oplus\ROT(\pi_h(\RR_j)), \zvel_0, p\big)$ as $\zonocg{c''}{G''}$, then $\xi(\RR_j,\zvel_0,p)$ is a zonotope because
\begin{equation}
    \pi_{xy}\big(\zonocg{c''}{G''}\big) = \left< \begin{bmatrix} [c'']_1 \\ [c'']_2 \end{bmatrix},~ \begin{bmatrix}[G'']_{1:} \\ [G'']_{2:}\end{bmatrix} \right>.
\end{equation}

Note $\pi_{xy}\big(\ROT(\pi_h(\RR_j))\big)=\rot(\pi_h(\RR_j))$, and by using the definition of $\pi_{xy}$ one can check that $\pi_{xy}(\mathcal A_1\oplus\mathcal A_2) = \pi_{xy}(\mathcal A_1)\oplus\pi_{xy}(\mathcal A_2)$ for any zonotopes $\mathcal A_1, \mathcal A_2\subset \R^{9+n_p}$.
Then by Lemma \ref{lem: slice and rot},
\begin{equation}
\label{eq: xi = pi_slice + rot}
    \xi(\RR_j,\zvel_0,p) = \pi_{xy}\big(\slice(\RR_j, \zvel_0, p)\big)\oplus\rot(\pi_h(\RR_j)).
\end{equation}

By Theorem \ref{thm: slice_is_good} for any $t\in T_j$ and $j\in\J$,  $\zaug(t)\in\slice(\RR_j,\zvel_0,p)\subset\RR_j$, then $h(t)\in\pi_h(\RR_j)$.
Because $\rot(\pi_h(\RR_j))$ by construction outer approximates the area over which $\Oego$ sweeps according to all possible heading of the ego vehicle during $T_j$,  then $\xi(\RR_j,\zvel_0,p)$ contains the vehicle footprint oriented according to $\pi_h(\RR_j)$ and centered at $\pi_{xy}(\zaug(t))$ during $T_j$.
\end{proof}
\section{Proof of Theorem \ref{thm:constraint}}
\label{app:constraint}

We first prove a pair of lemmas. 
The first lemma simplifies the expression of $\xi(\RR_j,z_0,p)$.

\begin{lem}
\label{lem: xi as a zono}
    Let $\RR_j= \zonocg{c_{\RR_j}}{[g_{\RR_j,1},g_{\RR_j,2},\ldots,g_{\RR_j,\ell_j}]}$ be the zonotope computed by CORA under the hybrid vehicle dynamics model $HS$ beginning from $\Zaug_0$ for arbitrary $j\in\J$, and let $\rot(\pi_h(\RR_j))=\zonocg{c_\rot}{G_\rot}$ be defined as \eqref{eq: def rot}.
    Then for arbitrary $\zvel_0\in\Zvel_0$ and $p\in\P$, there exist $c_\xi\in\W$, $A\in\R^{2\times n_p}$ and a real matrix $G_\xi$ with two rows such that
        $\xi(\RR_j,\zvel_0,p) = \zonocg{c_\xi+A\cdot p}{G_\xi}$.

\end{lem}
\begin{proof}
Recall $c^\text{slc}$ is defined as in \eqref{eq: slice center}, then
\begin{equation}
\label{eq: xi simplification}
    \begin{split}
        \hspace{-0.22cm} \xi(\RR_j,\zvel_0,p)   
                                        & =  \pi_{xy}\big( \zonocg{c^\text{slc}}{[g_{\RR_j,(3+n_p+1)},\ldots\\
                                        & \hspace{1.8cm}   \ldots,g_{\RR_j,\ell_j}]}  \big)\oplus \rot(\pi_h(\RR_j)) \\
                                        & = \zonocg{\pi_{xy}(c^\text{slc}) + c_\rot}{[\pi_{xy}(g_{\RR_j,(4+n_p)}),\ldots\\
                                        & \hspace{2.7cm}  \ldots,\pi_{xy}(g_{\RR_j,\ell_j}), G_\rot ]}.
    \end{split}
\end{equation}
where the first equality comes from using \eqref{eq: xi = pi_slice + rot} and \eqref{eq: slice def} and the last equality comes from denoting $\rot(\pi_h(\RR_j))$ as $\zonocg{c_\rot}{G_\rot}$ and performing Minkowski addition on two zonotopes.
$c^\text{slc}$ can be rewritten as
\begin{equation}
    \begin{split}
        c^\text{slc} = & ~ c_{\RR_j} + \sum_{k=7}^{9}\frac{[\zvel_0]_{(k-6)}-[c_{\RR_j}]_k}{[g_{\RR_j,(k-6)}]_k}g_{\RR_j,(k-6)} + \\
            &  - \sum_{k=10}^{9+n_p}\frac{[c_{\RR}]_k}{[g_{\RR,(k-6)}]_k}g_{\RR,(k-6)} + A'\cdot p 
    \end{split}
\end{equation}
with $A' = \left[ \frac{1}{[g_{\RR_j,4}]_{10}}g_{\RR_j,4}, \ldots, \frac{1}{[g_{\RR_j,(3+n_p)}]_{(9+n_p)}}g_{\RR_j,(3+n_p)} \right]$. 
Therefore by performing algebra one can find that $\xi(\RR_j,\zvel_0,p) = \zonocg{c_\xi+A\cdot p}{G_\xi}$ with some $c_\xi$, $G_\xi$ and
\begin{equation}
\label{eq: A def}
    \begin{split}
      A = &  \left[ \frac{1}{[g_{\RR_j,4}]_{10}}\pi_{xy}(g_{\RR_j,4}),\frac{1}{[g_{\RR_j,5}]_{11}}\pi_{xy}(g_{\RR_j,5}),\ldots\right.\\ 
    & \hspace{0.9cm}   \left.\ldots, \frac{1}{[g_{\RR_j,(3+n_p)}]_{(9+n_p)}}\pi_{xy}(g_{\RR_j,(3+n_p)}) \right].
    \end{split}
\end{equation}
\end{proof}

Note $\vartheta^\text{loc}(j,i,\zpos_0)$ is a zonotope by construction in \eqref{eq: vartheta local} because $\vartheta(j,i)$ is assumed to be a zonotope.
The following lemma follows from \cite[Lem. 5.1]{guibas2003zonotopes} and allows us to represent the intersection constraint in \opt.
\begin{lem}
\label{lem: set intersection is empty}
 Let $\xi(\RR_j,\zvel_0,p) = \zonocg{c_\xi+A\cdot p}{G_\xi}$ be computed as in Lemma \ref{lem: xi as a zono}, and let $\vartheta^\text{loc}(j,i,\zpos_0) = <c_\vartheta, G_\vartheta>$ be computed from Assumptions \ref{ass: obs in T} and \eqref{eq: vartheta local}.
 Then $\xi(\RR_j,\zvel_0,p)\cap\vartheta^\text{loc}(j,i,\zpos_0)\neq\emptyset$ if and only if 
     $A\cdot p\in \zonocg{c_\vartheta-c_\xi}{[G_\vartheta,G_\xi]}$.
\end{lem}

Now we can finally state the proof of Theorem \ref{thm:constraint}:

\begin{proof}
Let $\xi(\RR_j,\zvel_0,p) = \zonocg{c_\xi+A\cdot p}{G_\xi}$ as computed in Lemma \ref{lem: xi as a zono}, and let $\vartheta^\text{loc}(j,i,\zpos_0) = <c_\vartheta, G_\vartheta>$ be computed from Assumption \ref{ass: obs in T} and \eqref{eq: vartheta local}.
 Because all zonotopes are convex polytopes \cite{guibas2003zonotopes}, zonotope $\zonocg{c_\vartheta-c_\xi}{[G_\vartheta,G_\xi]}\subset \W\subseteq\R^2$ can be transferred into a half-space representation $\mathcal A := \{a\in\W\mid B\cdot a-b\leq0\}$ for some matrix $B$ and vector $b$.
 To find such $B$ and $b$, we denote $c = c_\vartheta-c_\xi\in\R^2$ and $G=[G_\vartheta,G_\xi]\in\R^{2\times\ell}$ with some positive integer $\ell$, and denote $B^- = \begin{bmatrix} -[G]_{2:} \\ [G]_{1:} \end{bmatrix}\in\R^{2\times\ell}$.
 Define
\begin{equation}
      \hspace{-0.2cm}  B^+:=\left[ \frac{[B^-]_{:1}}{\|[B^-]_{:1}\|} , \frac{[B^-]_{:2}}{\|[B^-]_{:2}\|} , \ldots , \frac{[B^-]_{:\ell}}{\|[B^-]_{:\ell}\|} \right]^\top \in\R^{\ell\times 2}.
\end{equation}
Then as a result of \cite[Thm 2.1]{althoff2010reachability}, $\zonocg{c}{G}=\{a\in\W\mid B\cdot a-b\leq0\}$ with
\begin{align}
    B &= \begin{bmatrix} B^+ \\ -B^+ \end{bmatrix}\in\R^{2\ell\times 2}, \label{eq:Bdef} \\ 
     b &= \begin{bmatrix} B^+ \cdot c+|B^+ \cdot G| \cdot \mathbf{1} \\ -B^+ \cdot c+|B^+ \cdot G| \cdot \mathbf{1} \end{bmatrix}\in\R^{2\ell} \label{eq:bdef}
\end{align}
where $\mathbf{1}\in\R^\ell$ is the column vector of ones.

By Lemma \ref{lem: set intersection is empty}, $\xi(\RR_j(d),\zvel_0,p)\cap\vartheta^\text{loc}(j,i,\zpos_0)=\emptyset$ if and only if $A\cdot p\notin \zonocg{c_\vartheta-c_\xi}{[G_\vartheta,G_\xi]}$, or in other words $A\cdot p\notin \mathcal A$. 
Notice $A\cdot p\notin \mathcal A$ if and only if $\max(B\cdot A\cdot p-b)>0$.

The subgradient claim follows from \cite[Theorem 5.4.5]{polak2012optimization}.
\end{proof}

\end{document}